\def\isarxivversion{1} %

\ifdefined\isarxivversion
\documentclass[11pt]{article}

\usepackage[numbers]{natbib}

\else
\documentclass[twoside]{article}

\usepackage{aistats2023}
\fi

\usepackage{natbib}

\usepackage{amsmath}
\usepackage{amsthm}
\usepackage{amssymb}
\usepackage{algorithm}
\usepackage{subfig}
\usepackage{algpseudocode}
\usepackage{graphicx}
\usepackage{grffile}
\usepackage{wrapfig,epsfig}
\usepackage{url}
\usepackage{xcolor}
\usepackage{epstopdf}

\usepackage{bbm}
\usepackage{dsfont}

\allowdisplaybreaks

\ifdefined\isarxivversion

\usepackage{tikz}
\usepackage{hyperref}  %
\hypersetup{colorlinks=true,citecolor=blue,linkcolor=blue} %
\usetikzlibrary{arrows}
\usepackage[margin=1in]{geometry}

\else

\usepackage{microtype}
\usepackage{hyperref}
\definecolor{mydarkblue}{rgb}{0,0.08,0.45}
\hypersetup{colorlinks=true, citecolor=mydarkblue,linkcolor=mydarkblue}

\fi

\newtheorem{theorem}{Theorem}[section]
\newtheorem{lemma}[theorem]{Lemma}
\newtheorem{definition}[theorem]{Definition}

\newtheorem{remark}[theorem]{Remark}
\newtheorem{claim}[theorem]{Claim}

\newcommand{\wt}{\widetilde}

\newcommand{\N}{\mathcal{N}}
\newcommand{\R}{\mathbb{R}}

\renewcommand{\tilde}{\wt}

\newcommand{\err}{\mathrm{err}}

\DeclareMathOperator*{\E}{{\mathbb{E}}}

\DeclareMathOperator{\nnz}{nnz}

\DeclareMathOperator{\sgn}{sgn}
\DeclareMathOperator{\diag}{\mathrm{diag}}
\newcommand{\Tinit}{{\cal T}_{\mathsf{init}}}
\newcommand{\Tquery}{{\cal T}_{\mathsf{query}}}
\newcommand{\Tupdate}{{\cal T}_{\mathsf{update}}}

\makeatletter
\newcommand*{\RN}[1]{\expandafter\@slowromancap\romannumeral #1@}
\makeatother

\newcommand{\blue}[1]{{\color{blue}#1}}

\usepackage{lineno}

\begin{document}

\title{Bypass Exponential Time Preprocessing: Fast Neural Network Training via Weight-Data Correlation Preprocessing}

\ifdefined\isarxivversion

\date{}

\author{
Josh Alman\thanks{\texttt{josh@cs.columbia.edu} Columbia University.} 
\and
Jiehao Liang\thanks{\texttt{jiehao.liang@berkeley.edu}. University of California, Berkeley.}
\and 
Zhao Song\thanks{\texttt{zsong@adobe.com} Adobe Research.} 
\and 
Ruizhe Zhang\thanks{\texttt{ruizhe@utexas.edu}. The University of Texas at Austin.}
\and 
Danyang Zhuo\thanks{\texttt{danyang@cs.duke.edu}. Duke University.}
}

\else

\twocolumn[

\aistatstitle{Bypass Exponential Time Preprocessing: Fast Neural Network Training via Weight-Data Correlation Preprocessing}

\aistatsauthor{ Author 1 \And Author 2 \And  Author 3 }

\aistatsaddress{ Institution 1 \And  Institution 2 \And Institution 3 } ]

\fi

\ifdefined\isarxivversion
\begin{titlepage}
  \maketitle
  \begin{abstract}
Over the last decade, deep neural networks have transformed our society, and they are already widely applied in various machine learning applications. State-of-art deep neural networks are becoming larger in size every year to deliver increasing model accuracy, and as a result, model training consumes substantial computing resources and will only consume more in the future.
Using current training methods, in each iteration, to process a data point $x \in \mathbb{R}^d$ in a layer, we need to spend $\Theta(md)$ time to evaluate all the $m$ neurons in the layer. This means processing the entire layer takes $\Theta(nmd)$ time for $n$ data points. Recent work [Song, Yang and Zhang, NeurIPS 2021] reduces this time per iteration to $o(nmd)$, but requires exponential time to preprocess either the data or the neural network weights, making it unlikely to have practical usage. 

In this work, we present a new preprocessing method that simply stores the weight-data correlation in a tree data structure in order to quickly, dynamically detect which neurons fire at each iteration. Our method requires only $O(nmd)$ time in preprocessing and still achieves $o(nmd)$ time per iteration. We complement our new algorithm with a lower bound, proving that assuming a popular conjecture from complexity theory, one could not substantially speed up our algorithm for dynamic detection of firing neurons.

  \end{abstract}
  \thispagestyle{empty}
\end{titlepage}

\else

\begin{abstract}
Over the last decade, deep neural networks have transformed our society, and they are already widely applied in various machine learning applications. State-of-art deep neural networks are becoming larger in size every year to deliver increasing model accuracy, and as a result, model training consumes substantial computing resources and will only consume more in the future.
Using current training methods, in each iteration, to process a data point $x \in \mathbb{R}^d$ in a layer, we need to spend $\Theta(md)$ time to evaluate all the $m$ neurons in the layer. This means processing the entire layer takes $\Theta(nmd)$ time for $n$ data points. Recent work [Song, Yang and Zhang, NeurIPS 2021] reduces this time per iteration to $o(nmd)$, but requires exponential time to preprocess either the data or the neural network weights, making it unlikely to have practical usage. 

In this work, we present a new preprocessing method that simply stores the weight-data correlation in a tree data structure in order to quickly, dynamically detect which neurons fire at each iteration. Our method requires only $O(nmd)$ time in preprocessing and still achieves $o(nmd)$ time per iteration. We complement our new algorithm with a lower bound, proving that assuming a popular conjecture from complexity theory, one could not substantially speed up our algorithm for dynamic detection of firing neurons.

\end{abstract}

\fi

\section{Introduction}\label{sec:intro}
Machine learning applications are requiring larger and larger neural network size, and the computing resources required to train these large models has also grown correspondingly. Determining how to train these models quickly has become an important research challenge.

Training a neural network is an iterative algorithm, and in each iteration we need to process each of the $m$ neurons on each of the $n$ data points. Assuming each data point has a length of $d$ (e.g., $d$ could be the size of an input image), this means the per-iteration training time of the straightforward algorithm is at least $\Omega(nmd)$ just to compute the activations. As we train larger neural networks on more training data, this running time can become a significant obstacle. 

Recent work by Song, Yang and Zhang~\cite{syz21} gave the first training algorithm that reduces this per iteration training time to $o(nmd)$. The high-level idea of their algorithm is to use a nearest neighbor search data structure which stores the neural network weights and training data. This allows the training method to have fast access to the inner products of the training data with the current weight of the iteration. However, their algorithm's initial preprocessing time to set up the data structure is \textit{exponential} in the dimension $d$, making it too slow in most applications. This raises a natural \textit{theoretical} question:

\begin{center}
    {\it Is it possible to design an algorithm which spends polynomial time to preprocess the weights and data, and which achieves a training time of $o(nmd)$ per iteration?}
\end{center}

This question is important for two reasons. First, speeding up neural network training is a fundamental research challenge with real-world value. Second, dynamic data structures have been successfully used to speed up computations in many contexts throughout computer science, yet their power and limitations when applied to the training of neural networks is currently poorly understood.

\subsection{Our Result: An Upper Bound}

Our main result answers this question in the affirmative, giving a new algorithm with efficient preprocessing and faster training in the natural over-parameterization regime (which has $m \gg n$):

\begin{theorem}[Main result]\label{thm:main}
There is a data structure which preprocesses $n$ data points in $d$-dimensional space, and $m$ initialization weights points in $d$-dimensional space, in $O(mnd)$ preprocessing time and $O(mn + md + nd)$ space, which can be used to speed up neural network training: Running the gradient descent algorithm on a two-layer, $m$-width, over-parameterized ReLU neural network, which will minimize the training loss to zero, can be performed with an expected running time (of the gradient descent algorithm per iteration) of 
\begin{align*}
    \wt{O}(m^{4/5} n^2 d).
\end{align*}
\end{theorem}
\begin{remark}
The prior work~\cite{syz21} presented two algorithms. Their first result (see Theorem 6.1 and Part 1 of Corollary B.6 in \cite{syz21}) has $O(2^d)$ preprocessing time, and uses $O(m^{1-1/d} n d)$ cost per iteration. Their second result (see Theorem 6.2 and Part 2 of Corollary B.6 of \cite{syz21}) has $O(n^d)$ preprocessing time, and uses $O(m^{4/5} nd)$ time per iteration. 
\end{remark}

Our key observation is that in each iteration of the training process, the weight updates are mostly sparse, and only a small fraction of neurons are activated for each training data point. Given this observation, we construct a binary search tree for each training data point (or neuron) to detect which neurons will fire. Our data structure and the corresponding algorithms are \textit{deterministic}, not relying on any randomness, and solve the following dynamic algorithms problem which we prove appears as a key subroutine of the training process.

\begin{definition}[Dynamic Detection of Firing Neurons (DDFN)]\label{def:def_dynamic_prob}
Given two set of points $X=\{x_1,\dots,x_n\}\subset \mathbb{Z}^d$, $Y=\{y_1,\dots,y_m\}\subset \mathbb{Z}^d$ and a threshold $b \in \mathbb{R}$, design a data structure to support the following operations:
\begin{itemize}
    \item \textsc{Update}$(j \in [m], z \in \mathbb{Z}^d )$, set $y_j$ to $z$
    \item \textsc{Query}$()$, either output the set
    \begin{align*}
    Q = \{ (i,j) \in [n] \times [m] ~|~ \langle x_i, y_j \rangle \geq b \},
    \end{align*}
    or report that $|Q|>m^{4/5}n$.
    \end{itemize}
\end{definition}

We give a data structure for DDFN which takes $O(mnd)$-time for preprocessing, $\wt{O}(nd)$-time per update, %
and $O(\min\{|Q|, m^{4/5}n\})$-time per query. %
At a high level, our data structure works as follows.

\paragraph{Preprocessing} 
We build $n$ binary search trees to maintain the $(x_i,y_j)$ pairs for $i\in [n]$ and $j\in [m]$. More specifically, the $i$-th tree has $m$ leaf nodes, storing the inner-products between $x_i$ and $\{y_j\}_{j\in[m]}$. Each internal node stores the larger value of its two child nodes. The preprocessing time for the binary search trees for all the input data and neurons takes $O(nmd)$ time and $O(mn)$ space.

\paragraph{Update} 
Suppose we will update $y_j$. Then, for each $i\in [n]$, we need to update a path from the leaf node corresponding to $(x_i,y_i)$ to the root, which contains $O(\log m)$ nodes. Hence, the running time of each update is $O(nd\log m)$.

\paragraph{Query}
We need to find all the leaf nodes with value greater than $b$. We can traverse each tree from top to bottom. At each node, if its value is at most $b$, we will not move further. Otherwise, we will try to search each of its child nodes. Note that the number of visited nodes is of the same order as the number of visited leaf nodes. And we visit a leaf if and only if its value is greater than $b$. Hence, the total query cost is $O(\min\{|Q|, m^{4/5}n\})$.

\subsection{Our Result: A Lower Bound}

We complement our new algorithm with a lower bound, showing that assuming a popular conjecture from complexity theory, one could not improve much on our running time for Dynamic Detection of Firing Neurons (DDFN). Prior work~\cite{syz21} got around this by using \emph{exponential} preprocessing time to avoid needing a dynamic algorithm. However, in our setting with polynomial preprocessing and running times, there is a limit to how quickly one can perform each iteration:

\begin{theorem}[Lower Bound for DDFN, informal version of Theorem~\ref{thm:lower_bound_formal}]\label{thm:lower_bound_informal}
Let $d = 2^{O(\log^* n)}$, and assume the $\mathsf{OVC}$ or $\mathsf{SETH}$ conjecture from complexity theory. For every constant $\varepsilon > 0$, there is no data structure for DDFN with $O(m^{4/5} n^{1/5 - \varepsilon})$ update time and $O(m^{4/5}n^{6/5 - \varepsilon})$ query time.
\end{theorem}

Here, $\log^* n$ denotes the \emph{iterated logarithm function}, which grows incredibly slowly, such that the dimension $2^{O(\log^* n)}$ is barely larger than a constant. The complexity-theoretic assumptions \textsf{OVC} and \textsf{SETH} are defined in Section~\ref{app:lower_bound}. We prove Theorem~\ref{thm:lower_bound_informal} by reducing the Maximum Inner Product Search problem to DDFN.

In other words, our Theorem~\ref{thm:lower_bound_informal} shows that it is impossible to substantially improve on our algorithm, no matter how sophisticated the algorithmic techniques one might hope to use.

\subsection{Related Work}

\paragraph{Orthogonal Vector Conjecture}

The orthogonal vector problem (\textsf{OV}) is a fundamental problem in fine-grained complexity which asks, given $X,Y \subset \{0,1\}^d$ of size $|X| = |Y| = n$, whether there are $x \in X$ and $y \in Y$ with $\langle x,y \rangle = 0$. The state-of-the-art algorithm \cite{awy14,cw16} runs in time $n^{2-1/O(\log c)}$ in dimension $d=c\log n$ for all $c\geq 1$; as the dimension $d$ increases, its running time approaches the trivial bound $n^2$. The orthogonal vector conjecture (\textsf{OVC}) conjectures an $n^{2-o(1)}$ lower bound for \textsf{OV} when $d=\omega(\log n)$. It is also known that the popular Strong Exponential Time Hypothesis (\textsf{SETH}) regarding the hardness of $k$-SAT implies \textsf{OVC}. This conjecture has been used to obtain conditional lower bounds for other important problems with polynomial-time algorithms in a wide variety of areas, including pattern matching \cite{avw14,bri14,bi15,bi16,bm16,bgl17,bk18,cw19}, graph algorithms \cite{rv13,abh16,gikw17,kt18,dlv22,cvx22}, and computational geometry \cite{bbk16,rub18,wil18,c20,km20}; see also the survey~\cite{williams2018some}.

\paragraph{Acceleration via high-dimensional search data-structure.}

Data structures have been designed which allow one to quickly find high-dimensional points in geometric query regions (e.g., half-spaces, simplices, etc). Currently, there are two main approaches to designing these structures. One is based on Locality Sensitive Hashing (LSH) \cite{im98}, which aims to find nearby points (e.g., small $\ell_2$ distance \cite{diim04,ar15,ailrs15,arm17,r17,air18,biw19,dirw20} or large inner product \cite{sl14,sl15_uai,sl15_www}) to a query $q\in \mathbb{R}^d$ in a given set of points $S\subset \mathbb{R}^d$. LSH-based algorithms typically run quickly in practice, but only support approximate nearest neighbor queries. The other approach is based on space partitioning data structures, such as partition trees \cite{m92,m92b,aem92,ac09,cha12}, $k$-$d$ trees / range trees \cite{c17,tog17,c19}, and Voronoi diagrams \cite{adms98, c00}, which can exactly search for points in the queried region.

There is a recent line of research which has applied high-dimensional geometric data structures to reduce deep neural networks' training time in practice. Empirically, SLIDE \cite{cmf+20} uses LSH-based methods to efficiently find neurons with maximum inner product in the forward pass; Reformer \cite{kkl20} also applies LSH to save the space complexity such that the neural networks are able to process very long sequences; MONGOOSE \cite{clp+21} combines the learnable LSH-based data structure \cite{c02} with the lazy update framework \cite{cls19} to speedup the neural network training. Theoretically, \cite{syz21} gives the first provable sublinear-time training algorithm for 2-layer over-parameterized neural networks using the HSR data structures in \cite{aem92}.  

The goal of our paper is to to design an efficient high-dimensional geometric data structure that can be embedded in the neural network training framework with provable performance guarantees. Specifically, our data structures have the same functionalities as the HSR data structures \cite{aem92} which can find all the points that have large inner products and support efficient data update. However, our data structures do not have an exponential dependence on $d$, the dimension of the points, which appears in the time complexities of many computational geometry algorithms (including \cite{aem92}) due to the curse of dimensionality. Compared with the LSH-based approaches, our data structures have a stronger correctness guarantee, and will always report all the points with a sufficiently large inner product; LSH only gives approximate guarantees and might miss some of them.  %

\paragraph{Convergence via over-parameterization.} Over-parameterization, where the trainable parameters are much larger than the number of training data points (i.e., $m\gg n$), is a very natural and common regime in deep learning. It plays a key role in explaining why deep neural networks can perform so well at many different tasks. Over the last few years, there has been a tremendous amount of work toward theoretically understanding the convergence and generalization of deep neural networks in the over-parameterization regime, e.g.,~\cite{ll18,dzps19,als19_dnn,als19_rnn,adhlw19,adhlsw19,sy19,cgh+19,zmg19,cg19,zg19,os19,jt20,lsswy20,hlsy21,zpdlsa20,bpsw21,szz22,z22}. A key observation is that when the width ($m$) of the neural network tends to infinity, the neural network is equivalent to a neural tangent kernel (NTK) \cite{jgh18}, and so technical tools from kernel methods can be adapted to analyze deep neural networks.   
In particular, it has been shown that (stochastic) gradient descent ((S)GD) can train a sufficiently wide neural network with random initialization, converging to a small training error in a polynomial number of steps.

\paragraph{Roadmap} This paper is organized as follows: In Section~\ref{sec:preli}, we formulate our problem of training neural networks. In Section~\ref{sec:correlation_tree}, we develop the Correlation Tree data structure, which is the main contribution of this work. In Section~\ref{sec:running_time_of_main_result}, we state our main result for quickly training neural network using Correlation Trees. Section~\ref{app:lower_bound} shows a formal version of lower bound for dynamic detection of firing neurons. In Section~\ref{sec:conclusion}, we conclude and discuss some future directions.  A number of our proofs are deferred to the appendix. %

\section{Preliminaries}\label{sec:preli}

Before describing our new data structure, we first present the notation we will use, and formulate the problem setting.  
\paragraph{Basic Notation.}
For $n\in \mathbb{N}_+$, we use $[n]$ to denote the set $\{1,2,\cdots,n\}$. We write $\E[X]$ to denote the expected value of a random variable $X$, and $\Pr[Y]$ to denote the probability of a random event $Y$. For a matrix $M$, we write $M^\top$ to denote the transpose of $M$. We use $x^\top y$ to denote the inner product between vectors $x$ and $y$. We use $I_d$ to denote the  $d \times d$ identity matrix. We use $\mathcal{N}(\mu; \sigma^2)$ to denote
the Gaussian distribution with mean $\mu$ and variance $\sigma^2$.

\paragraph{Problem Formulation}

In this section, we introduce the neural network model we study in this paper. We consider a two-layer ReLU-activated neural network $f$ that has width $m$ and uses an $\ell_2$ loss function.

\begin{definition}[Prediction function and loss function]\label{def:neural_network}
For a threshold parameter $b\in\R$, data point $x\in\R^d$, weight matrix $W\in\R^{d\times m}$, and weights $a \in \R^m $, the prediction function $f(W,x,a)$ and loss function $L(W)$ are given by
\begin{align*}
    f(W,x,a):=\frac{1}{\sqrt{m}}\sum_{r=1}^m a_r\sigma_b(\langle w_r, x\rangle),\\
    L(W):=\frac{1}{2}\sum_{i=1}^n (f(W,x_i,a)-y_i)^2,
\end{align*}
where $\sigma_b(x):=\max\{x-b,0\}$ is the ReLU function with threshold parameter $b$.
Following the prior work~\cite{syz21}, we write $2\mathrm{NN}(m,b)$ to denote this function $f$ for simplicity.
\end{definition}
\begin{remark}
In the neural network, $w_1,\cdots,w_m\in \R^d$ are the weight vectors of the edges between input nodes and hidden neuron nodes, and $a_1,\cdots,a_m \in \R$ are the weights of the edges connecting hidden neuron nodes with the output node. 

Following the setups in previous work, we only train the weight parameters $W\in\R^{d\times m}$ to minimize the loss $L(W)$, and will leave $a\in\R^m$ unchanged  after the initialization.
\end{remark}

In this work, we study the following training process:
\begin{itemize}
    \item {\bf Initialization: } For each hidden neuron, we sample $w_r(0)\sim \mathcal{N}(0,I_d)$, and sample $a_r$ from $\{-1,+1\}$ uniformly at random.
    \item {\bf Gradient computation: } For each neuron, we have
    \begin{align*}
        \frac{\partial f(W,x,a)}{\partial w_r}=&~ \frac{a_r}{\sqrt{m}}x {\bf 1
        }_{w_r^\top x  \geq b},~\text{and}\\
        \frac{\partial L(W)}{\partial w_r}=&~ \frac{a_r}{\sqrt{m}}\sum_{i=1}^n (f(W,x_i,a)-y_i) x_i {\bf 1}_{\langle w_r,x_i \rangle \geq b}.
    \end{align*}
    \item {\bf Weight update: } We follow the standard update rule of the GD algorithm from iteration $k$ to iteration $k+1$:
    \begin{align*}
        W(k+1)=W(k)-\eta\cdot\delta W(k),
    \end{align*}
    where $W(k)$ denotes the weights at iteration $k$, and 
    \begin{align*} 
    \delta W(k)=\frac{\partial L(W(k))}{\partial W(k)}.
    \end{align*}
\end{itemize}

\paragraph{Sparsity phenomenon in the training process}

As observed in many experimental and theoretical works \cite{chen2020mongoose_no,syz21,szz22}, for a randomly initialized over-parameterized neural network (Definition~\ref{def:neural_network}), given input data $x$, only a small fraction ($o(m)$) of the neurons will be activated when evaluating $2\mathrm{NN}(m,b)$ (Definition~\ref{def:neural_network}) in each training iteration. We refer to this phenomenon as ``Sparsity''. We exploit this property to design an efficient data structure to identify the sparse activated neurons, achieving sublinear training time in terms of $m$, the number of neurons in the hidden layer.

To be specific, the sparsity of activated neurons during the training process is bounded by choosing a proper threshold $b$ for the ReLU %
function. Because of a concentration phenomenon of the randomized initialization, we can bound the number of activated neurons just after initialization, which we refer as ``sparsity after initialization''. Then, in subsequent training iterations, using the neural tangent kernel (NTK) property, it follows that there is only a minor increase in the number of activated neurons per iteration. Therefore, the total number of activated neurons can be bounded by a small quantity.

For simplicity, we define the ``fire set''  ${\cal S}_\mathrm{fire}(x)$ first, which is the set of neurons that is activated when the neural network's input is $x$.
\begin{definition}[Fire set]\label{def:fire_set}
Let the neural network be defined as in Definition~\ref{def:neural_network}. For a data  point $x\in\R^d$, let $\mathcal{S}_{\mathrm{fire}}(x)$ denote the set of neurons that are activated on input $x$, i.e.,
\begin{align*}
    \mathcal{S}_{\mathrm{fire}}(x):=\{i\in[m]:\sigma_b(\langle x,w_i\rangle) >0\}.
\end{align*}

\end{definition}

Then, we similarly define fire sets for hidden neurons and input data points for each iteration:
\begin{definition}[Fire set per iteration]\label{def:fire_set_per_iter}
For each data point $x_i\in\R^d$ with $i\in[n]$ and each iteration $t\in\{0,1,\cdots,T\}$, let $w_r(t)\in\R^d$ be the weight vector of the $r$-th neuron at the $t$-th iteration for $r\in [m]$. Define %
\begin{align*}
    S_{i,\mathrm{fire}}(t)&:=\{r\in[m]:\sigma_b(\langle x_i,w_r(t) \rangle) > 0\}, \\
    \tilde{S}_{r,\mathrm{fire}}(t)&:=\{i\in[n]:\sigma_b(\langle x_i,w_r(t) \rangle) > 0\}.
\end{align*}
We further denote the sizes of these sets by $k_{i,t}:=|S_{i,\mathrm{fire}}(t)|$ and $\tilde{k}_{r,t}:=|\tilde{S}_{r,\mathrm{fire}}(t)|$.
\end{definition}

The following lemma upper bounds  the sparsity after initialization. 

\begin{lemma}[Sparsity after initialization, informal version of Lemma~\ref{lem:sparsity_formal}, \cite{syz21}]\label{lem:sparsity}
Let $b>0$ be a tunable parameter. If we setup the neural network as in Definition~\ref{def:neural_network}, then after the randomized initialization, with probability at least $1-\exp(-\Omega(m\cdot \exp(-b^2/2)))$, it holds that for any input data $x$, the number of activated neurons is at most $O(m\cdot \exp(-b^2/2))$, where $m$ is the total number of neurons.
\end{lemma}

\begin{remark}\label{rmk:sparsity}
This suggests that if we take $b=\sqrt{0.4 \log m}$, we achieve a sublinear number, $O(m^{4/5})$, of activated neurons.
\end{remark}

We can similarly control the sparisity in each iteration, and not just the first iteration; we defer the details to Section~\ref{app:sparsity_per_iteration}.

In the next section, we will show how our weight-tree correlation data structure can take advantage of this sparsity phenomenon.

\section{Correlation Tree Data Structure}\label{sec:correlation_tree}

In this section, we consider a neural network $2\mathrm{NN}(m,b)$ (Definition~\ref{def:neural_network}) with $n$ data points. We let $\{w_1, \cdots, w_m\} \subset \R^d$ be the weights, $\{x_1, \cdots, x_n\} \subset \R^d$ be the data points, and 
\begin{align*}
\{(w_r, x_i)\}_{r \in [m], i \in [n]} \subset \R^{m+n}
\end{align*}
be the weight-data pairs.

We propose two data structures: Correlation DTree and Correlation WTree. The DTree data structure has $n$ trees, and its $i$-th tree has $m$ leaf nodes corresponding to the set of inner-products between $x_i$ and all hidden neurons, i.e.,  $\{\langle w_r, x_i \rangle\}_{r \in [m]}$. Similarly, the WTree data structure consists of $m$ trees, and its $r$-th tree has $n$ leaf nodes corresponding to the set of inner-products between the $r$-th neuron and all data points, i.e., $\{\langle w_r, x_i \rangle\}_{i \in [n]}$. %

The Correlation Tree is a simple binary tree data structure. At a high level, it works as follows:
\begin{itemize}
    \item \textbf{Tree construction } We first calculate the inner-products of all weight-data pairs $\langle w_i, x_j\rangle$, each representing the evaluation of a neuron at a data point. To search activated neurons efficiently, we create a tree structure in the following way (taking the Correlation DTree as an example): %
    we first build $m$ leaf nodes, where the $r$-th leaf stores $\langle w_r, x_i \rangle$ for  $r \in [m]$. Then, we recursively construct a binary tree such that a parent node takes the larger value from its two child nodes. Finally, we obtain a tree with root having value $\max_{r \in [m]}\{\langle w_r, x_i \rangle\}$. Moreover, the value of each internal node equals to the maximum value among the leaf nodes in this subtree.

    \item \textbf{Efficient search } Given a threshold $b$, the data structure can find all the pairs of vectors whose inner product is greater than $b$. Take the Correlation DTree as an example. It outputs the indices of those activated neurons (i.e., $\langle w_r, x_i \rangle > b$) in a recursive way: starting from the root, it checks whether it is ``activated'' (i.e., with value $>b$). If not, the search ends. Otherwise, it moves to each of the child nodes and repeat this searching process until stop. This is a typical depth-first search strategy. Its running time is determined by how many nodes it visits during searching. The number of visited nodes have the same magnitude as the number of visited leaf nodes, i.e., the number of activated neurons. Hence, the efficiency of our data structures relies on the sparsity phenomenon of the training process.
    
    \item \textbf{Relation between DTree and WTree} In the Correlation DTree, each weight vector $w_r$ appears only in $n$ different trees. In the Correlation WTree, each weight vector $w_r$ appears only in one of the $m$ trees. When $w_r$ is updated, DTree will change the nodes along a root-to-leaf path in $n$ trees, whereas WTree only changes such paths in the $r$-th tree.
\end{itemize}

\subsection{Correlation DTree data structure}
We now state our main theorem summarizing the correlation DTtree data structure. Its pseudocode is given in Algorithms~\ref{alg:correlation_tree_init_intro} and \ref{alg:correlation_tree_query_data_intro} below. Its proof are deferred to Section~\ref{app:dtree}.
\begin{theorem}[Correlation DTree data structure]\label{thm:correlation_tree_data_structure}
There exists a data structure with the following procedures:
\begin{itemize}
    \item \textsc{Init}$(\{w_1,w_2, \cdots, w_m\} \subset \R^d, \{x_1, x_2, \cdots, x_n\} \subset \R^d,n\in\mathbb{N},m\in\mathbb{N},d\in\mathbb{N})$. Given a series of weights $w_1,w_2,\cdots,w_m$ and data $x_1, x_2, \cdots, x_n$ in d-dimensional space, it performs preprocessing in time $O(nmd)$.
    \item \textsc{Update}$(z\in\R^d,r\in [m])$. Given a weight $z$ and an index $r$, it updates weight $w_r$ to $z$ in time $O(n\cdot(d+\log m))$.
    \item \textsc{Query}$(i \in [n],\tau \in \R)$. Given an index $i$ indicating data point $x_i$ and a threshold $\tau$, it finds all indices $r\in[m]$ such that $\langle w_r,x_i \rangle >\tau$ in time $O(|\tilde{S}(\tau)|\cdot \log m)$, where $$\tilde{S}(\tau):=\{r:\langle w_r,x_i \rangle >\tau\}.$$
\end{itemize}
\end{theorem} 

\begin{algorithm}[!ht]\caption{Correlation DTree data structure} \label{alg:correlation_tree_init_intro} 
\begin{algorithmic}[1]
\State {\bf data structure} \textsc{CorrelationDTree} \Comment{Theorem~\ref{thm:correlation_tree_data_structure}}
\State {\bf members}
\State \hspace{4mm} $W \in \R^{m\times d}$ ($m$ weight vectors )
\State \hspace{4mm} $X \in \R^{n \times d}$ ($n$ data points)
\State \hspace{4mm} Binary tree $T_1, T_2, \cdots, T_n$ \Comment{$n$ binary search trees}
\State {\bf end members}
\Procedure{Init}{$w_1,w_2, \cdots, w_m \in \R^d, m, x_1, x_2, \cdots, x_n \in \R^d$, $n$, $m$,  $d$}  \Comment{Lemma~\ref{lem:correlation_tree_init_formal}}
    \For{$i=1 \to n$} \label{lin:init_first_loop}
        \State $x_i \gets x_i$
    \EndFor
    \For{$j=1 \to m$} \label{lin:init_second_loop}
        \State $w_j \gets w_j$
    \EndFor
    \For{$i=1 \to n$} \Comment{for data point, we create a tree} \label{lin:init_outer_loop}
        \For{$j=1 \to m$} \label{lin:init_inner_loop}
            \State $u_j \gets \langle x_i, w_j \rangle$ \label{lin:init_inner_product}
        \EndFor
        \State $T_i \gets \textsc{MakeMaxTree}(u_1, \cdots, u_m)$\label{lin:init_make_binary_tree} \Comment{Each node stores the maximum value for his two children, Algorithm~\ref{alg:make_tree_app}} %
    \EndFor
\EndProcedure
\Procedure{Update}{$z\in\R^d, r \in [m]$} \Comment{Lemma~\ref{lem:correlation_tree_update_weight_formal}} 
\State $w_r \gets z$ 
\For{$i=1 \to n$} \label{lin:update_loop}
    \State $l \gets$ the $l$-th leaf of tree $T_i$ \label{lin:update_find_leaf}
    \State $l.\text{value} = \langle z, x_i \rangle$ \label{lin:update_inner_product}
    \While{$l$ is not root}
        \State $p$ $\gets$ parent of $l$
        \State $p.\text{value} \gets \max \{ p.\text{value}, l.\text{value} \}$
        \State $l \gets p$
    \EndWhile
\EndFor
\EndProcedure
\State {\bf end data structure}
\end{algorithmic}
\end{algorithm}
\begin{algorithm}[!ht]\caption{Correlation DTrees }\label{alg:correlation_tree_query_data_intro}
\begin{algorithmic}[1]
\State {\bf data structure} \textsc{CorrelationDTree} \Comment{Theorem~\ref{thm:correlation_tree_data_structure}} 
\Procedure{Query}{$i \in [n], \tau \in \R_{\geq 0}$} \Comment{Lemma~\ref{lem:correlation_tree_query_data_formal}} 
\State \Return \textsc{Find}($\tau,\mathrm{root}(T_i)$)
\EndProcedure
\Procedure{Find}{$\tau \in \R_{\geq 0}, r\in T$} %
\If{$r$ is leaf}
\State \Return $r$
\Else
\State $r_1\gets$ left child of $r$, $r_2\gets$ right child of $r$
\If{$r_1.\text{value} \geq \tau$}
    \State $S_1 \gets $\textsc{Find}$(\tau,r_1)$
\EndIf
\If{$r_2.\text{value} \geq \tau$}
    \State $S_2 \gets $\textsc{Find}$(\tau,r_2)$
\EndIf
\EndIf
\State \Return $S_1 \cup S_2$
\EndProcedure
\State {\bf end data structure}
\end{algorithmic}
\end{algorithm}

\subsection{Correlation WTree data structure}

We next state the main theorem summarizing our similar Correlation WTree data structure. Both the Correlation DTree and Correlation WTree have a query time which is roughly equal to the output size, but since they have different outputs, each can be faster than the other depending on the setting. The pseudocode and proof for Correlation WTree are deferred to Section~\ref{app:wtree}.

\begin{theorem}[Correlation WTree data structure]\label{thm:correlation_wtree_data_structure}
There exists a data structure with the following procedures:
\begin{itemize}
    \item \textsc{Init}$(\{w_1,w_2, \cdots, w_m\} \subset \R^d, \{x_1, x_2, \cdots, x_n\} \subset \R^d,n\in\mathbb{N},m\in\mathbb{N},d\in\mathbb{N})$. Given a series of weights $w_1,w_2,\cdots,w_m$ and data $x_1, x_2, \cdots, x_n$ in d-dimensional space, it performs preprocessing in time $O(nmd)$.
    \item \textsc{Update}$(z\in\R^d,r\in [m])$. Given a weight $z$ and index $r$, it updates weight $w_r$ to $z$ in time $O(nd)$.
    \item \textsc{Query}$(r \in [m],\tau \in \R)$. Given an index $r$ indicating weight $w_r$ and a threshold $\tau$, it finds all indices $i\in[n]$ such that $\langle w_r,x_i \rangle >\tau$ in time $O(|S(\tau)|\cdot \log m)$, where $$S(\tau):=\{i:\langle w_r,x_i \rangle >\tau\}.$$
\end{itemize}
\end{theorem}

\section{Running Time of Our Algorithm}\label{sec:running_time_of_main_result}
In this section, we show how to apply the Correlation Tree data structures developed in Section~\ref{sec:correlation_tree} to speed up neural network training. %
\subsection{Weights Preprocessing}

\begin{algorithm*}[!ht]
\caption{Training Neural Network based on Correlation DTree}
\label{alg:ds_for_w_training} 
\begin{algorithmic}[1]
    \algrenewcommand\algorithmicprocedure{\textbf{procedure}}
	\Procedure{TrainingWithDTree}{$\{(x_i,y_i)\}_{i\in [n]}$,$n$,$m$,$d$} \Comment{Theorem~\ref{thm:running_time_data}}
	\State Initialize $w_r, a_r$ for $r\in [m]$ and $b$ according to Section~\ref{sec:preli} %
    \State \textsc{DTree}.\textsc{Init}($\{w_r(0)\}_{r\in [m]}, m, d$) 
    \Comment{Algorithm~\ref{alg:correlation_tree_init_app}}

	\For{$t=1 \to T$}

	        \State $S_{i,\mathrm{fire}}\gets \textsc{DTree}.\textsc{Query}(x_i,b)$ for $i \in [n]$ 

		    \State Forward pass for $x_i$ only on neurons in $S_{i,\mathrm{fire}}$ for $i \in [n]$
		    \State Calculate gradient for $x_i$ only on neurons in $S_{i, \mathrm{fire}}$ for $i \in [n]$

            \State Gradient update for the neurons in $\cup_{i \in [n]}S_{i, \mathrm{fire}}$
            
            \State \textsc{DTree}.\textsc{Update}($w_r(t+1), r$)
	\EndFor
	\State \Return Trained weights $w_r(T + 1)$ for $r \in [m]$
	\EndProcedure
\end{algorithmic}
\end{algorithm*}

In Algorithm~\ref{alg:nn_dtree_app}, we use DTree structure to speed up the training process. We preprocess weights $w_r, r \in [m]$ for each data point $x_i, i \in [n]$ by constructing $n$ weight-data correlation trees. In each iteration, \textsc{Query} finds the set of activated neurons $S_{i,\mathrm{fire}}$ (Definition~\ref{def:fire_set_per_iter}) efficiently for each data point $x_i$ and \textsc{Update} helps change the weights in backward propagation. 

Our main result for weight preprocessing is as follows.

\begin{theorem}[Running time part, informal version of Theorem~\ref{thm:running_time_data_formal}]\label{thm:running_time_data}
Given $n$ data points in $\R^d$, gradient descent using the DTree data structure (Algorithm~\ref{alg:nn_dtree_app}) for the neural network $\mathrm{2NN}(m,b=\sqrt{0.4\log m})$ (Definition~\ref{def:neural_network}) takes $$O(m^{4/5}n^2d)$$ time per iteration in expectation.
 
\end{theorem}

\subsection{Data Preprocessing}

\begin{algorithm*}[!ht]
\caption{Training Neural Network based on Correlation WTree}
\label{alg:ds_for_x_training} 
\begin{algorithmic}[1]
    \algrenewcommand\algorithmicprocedure{\textbf{procedure}}
	\Procedure{TrainingWithWTree}{$\{(x_i,y_i)\}_{i\in [n]}$,$n$,$m$,$d$} \Comment{Theorem~\ref{thm:running_time_weight}}
	\State Initialize $w_r, a_r$ for $r\in [m]$ and $b$ according to Section~\ref{sec:preli} %
    \State \textsc{WTree}.\textsc{Init}($\{x_i\}_{i \in [n]}, n, d$) \Comment{Algorithm~\ref{alg:correlation_wtree_init}}
        \State $\wt{S}_{r,\mathrm{fire}} \leftarrow \textsc{wt}.\textsc{Query}(w_r(0),b)$ for $r\in [m]$  \Comment{Data points fire set} \label{ln:init_s_tilde}
        \State $S_{i, \mathrm{fire}} \leftarrow \{ r~|~i \in \wt{S}_{r, \mathrm{fire}} \}$ \Comment{Hidden neurons fire set} \label{ln:init_s}
	\For{$t=1 \to T$}
	   \State Forward pass for $x_i$ only on neurons in $S_{i,\mathrm{fire}}$ for $i \in [n]$ \label{ln:forward_b}
		    \State Calculate gradient for $x_i$ only on neurons in $S_{i, \mathrm{fire}}$ for $i \in [n]$
		\For{$r \in \cup_{i \in [n]}\mathcal{S}_{i, \mathrm{fire}}$} \label{ln:maintain_b}
		
        \State $\wt{S}_{r,\mathrm{fire}} \leftarrow \textsc{WTree}.\textsc{Query}(w_r(t+1),b)$  
        \State Update $S_{i,\mathrm{fire}}$ for each $i\in \wt{S}_{r,\mathrm{fire}}$
    \EndFor\label{ln:maintain_e}
	\EndFor
	\State \Return Trained weights $w_r(T + 1)$ for $r \in [m]$
	\EndProcedure
\end{algorithmic}
\end{algorithm*}

Preprocessing weights based on data points is a common practice for neural networks. Here we consider its dual form: preprocessing input data $x_i, i \in [n]$ based on neural network weights $w_r, r \in [m]$. This can be easily done due to the symmetric property of the inner product that we used in the correlation tree structure.

Given a weight vector $w_r$, we can quickly find $\tilde{S}_{i,\mathrm{fire}}$ (Definition~\ref{def:fire_set_per_iter}) which contains the indices of data points that ``fire'' for weight $w_r$. By the dual relationship between $\tilde{S}_{i,\mathrm{fire}}$ and ${S}_{i,\mathrm{fire}}$, we can recover ${S}_{i,\mathrm{fire}}$ easily.

One advantage of the data preprocessing approach is that the data structure only depends on the training dataset, instead of the neural network architecture. Therefore, the data structure could be pre-computed and stored in cloud platforms.%

The performance guarantee of our data preprocessing training algorithm is shown as follows:

\begin{theorem}[Running time part, informal version of Theorem~\ref{thm:running_time_weight_formal}]\label{thm:running_time_weight}
Given $n$ data points in $\R^d$, gradient descent algorithm using the WTree data structure (Algorithm~\ref{alg:nn_wtree_app}) for the neural network $\mathrm{2NN}(m,b=\sqrt{0.4\log m})$ takes $O(m^{4/5}n\cdot\log n)$-time per iteration to initialize  $\tilde{S}_{r,\mathrm{fire}},S_{i,\mathrm{fire}}$ for $r\in[m],i\in[n]$, and the total running time per iteration  is $$O(m^{4/5} n^2 d)$$ in expectation. 
\end{theorem}

\section{Lower Bound for Dynamic Detection of Firing Neurons}\label{app:lower_bound}

The goal of this section is to prove the lower bound for Dynamic Detection of Firing Neurons.

We start by introducing the strong exponential time hypothesis, \textsf{SETH} in abbreviation.

\begin{definition}[Strong exponential time hypothesis, \textsf{SETH}, \cite{ip01,cip09}]\label{def:seth}
For every $\epsilon > 0$, there exists a $k=k(\epsilon) \in \mathbb{N}$ such that no algorithm can solve $k$-SAT (i.e., satisfiability on a CNF of width $k$) in 
\begin{align*} 
O(2^{(1-\epsilon) n})
\end{align*}
time where $n$ is the number of variables.
\end{definition}

We present another relative concept called orthogonal vector conjecture, \textsf{OVC} in abbreviation.

\begin{definition}[Orthogonal vector conjecture, \textsf{OVC}, \cite{wil05,avw14,bi15,abv15}]\label{def:ovc}
For every $\epsilon> 0$, there exists a $c \geq 1$ such that the orthogonal vector problem of size $n$ in $d$-dimension requires $n^{2-\epsilon}$-time when $d = c\log n$.
\end{definition}

We refer to a theorem about maximum bichromatic inner product lower bound in \cite{c20}.

\begin{theorem}[Maximum bichromatic inner product lower bound, \cite{c20}]\label{thm:max_ip_lb}
Assuming \textsf{SETH} (Definition~\ref{def:seth}) or \textsf{OVC} (Definition~\ref{def:ovc}), there is a constant $c$ such that any exact algorithm for $\mathbb{Z}\textsf{-Max-IP}_{n,d}$ in dimension $d = c^{\log^*n}$ requires \begin{align*} 
n^{2-o(1)}
\end{align*}
time, with vectors of $O(\log n)$-bit entries.
\end{theorem}

Putting things together, we state the main result for the lower bound for Dynamic Detection of Firing Neurons.

\begin{theorem}[Lower Bound for Dynamic Detection of Firing Neurons, Formal version of Theorem~\ref{thm:lower_bound_informal}]\label{thm:lower_bound_formal}
Let 
\begin{align*} 
d = & ~ 2^{O(\log^* n)}, \\
m = & ~ \Theta(n).
\end{align*}
Unless $\mathsf{OVC}$ or $\mathsf{SETH}$ fails, for every constant $\varepsilon > 0$, no data structure can perform updates in time $O(n^{1 - \varepsilon})$ and answer queries in time $O(n^{2 - \varepsilon})$.
\end{theorem}

\begin{proof}  
Let $m=n$ and $d=c^{\log^* n}$, where $c$ is defined in Theorem~\ref{thm:max_ip_lb}.

Suppose there exists a data structure that for $(m=n,n,d+1)$-sized instance, can perform updates in $n^{1-\epsilon}$-time and answer queries in $n^{2-\epsilon}$-time, for some $\epsilon >0$. 

Let $X=\{x_1,\dots,x_n\}\subset \mathbb{Z}^d$, $Y=\{y_1,\dots,y_n\}\subset \mathbb{Z}^d$ be a hard instance of $\mathbb{Z}\textsf{-Max-IP}_{n,d}$ problem constructed in Theorem~\ref{thm:max_ip_lb}. For each vector $x_i$ (or $y_j$), we construct a new vector $\wt{x}_i$ (or $\wt{y}_j$) in $d+1$ dimension such that $(\wt{x}_i)_{d+1}=w$ and $(\wt{y}_j)_{d+1}=-1$, where $w$ is a parameter to be chosen later.

Then, we construct an instance of the problem in Definition~\ref{def:def_dynamic_prob} as follows:
\begin{align*} 
\wt{X}:=\{\wt{x}_1,\dots,\wt{x}_n\},
\end{align*} 
and
\begin{align*}
\wt{Y}:=\{\wt{y}_1,\dots,\wt{y}_n\},
\end{align*}
and $b=0$.

We show that the data structure for this instance $(\wt{X}, \wt{Y}, b)$ can be used to solve $\mathbb{Z}\textsf{-Max-IP}_{n,d}(X,Y)$.

We perform a binary search for the value of $\mathbb{Z}\textsf{-Max-IP}_{n,d}(X,Y)$. Note that at most $O(\log n)$ iterations suffice to find the exact answer.

Suppose the current value in the binary search is $t\in \mathbb{Z}$. We first call \textsc{Update}() to set $(\wt{x}_i)_{d+1}=t$ for each $i\in [n]$. By the data structure's guarantee, this step takes 
\begin{align*} 
O(n\cdot n^{1-\epsilon})=O(n^{2-\epsilon})
\end{align*}
time. 

Then, we call \textsc{Query}(). Notice that 
\begin{align*}
    \langle \wt{x}_i,\wt{y}_j\rangle = \langle x_i, y_j\rangle - t\geq 0 ~\Longleftrightarrow \langle x_i, y_j\rangle \geq t.
\end{align*}
Hence, \textsc{Query}() will return all pairs of $(i,j)$ such that $\langle x_i, y_j\rangle \geq t$. This step runs in $O(n^{2-\epsilon})$-time. And based on whether the output of  \textsc{Query}() is empty or not, we know the direction of the binary search for the next iteration.

Hence, each iteration of the binary search takes $O(n^{2-\epsilon})$-time. Thus, we can solve $\mathbb{Z}\textsf{-Max-IP}_{n,d}(X,Y)$ in
\begin{align*} 
O(n^{2-\epsilon}\cdot \log n) = O(n^{2-\epsilon'})
\end{align*}
time, for some constant $\epsilon'<\epsilon$. However, this contradicts to the lower bound for $\mathbb{Z}\textsf{-Max-IP}_{n,d}$ in Theorem~\ref{thm:max_ip_lb}.

Therefore, no such data structure can exist.
\end{proof}
\vspace{-2mm}
\section{Conclusion}\label{sec:conclusion}
\vspace{-2mm}
Deep neural networks are becoming larger every year to offer improved model accuracy. Training these models consumes substantial resources, and resource consumption will only increase as these models grow. In traditional training methods, for each iteration, we need to spend $\Theta(nmd)$ time to evaluate the $m$ neurons on $n$ data points with dimension $d$. Recent work~\cite{syz21} reduced the per-iteration cost to $o(nmd)$, but required exponential time to preprocess either the data or the neural weights. We develop a new method that reduces the preprocessing cost to $O(nmd)$ while keeping the per-iteration running time at $o(nmd)$. In particular, we design a simple binary tree-based dynamic geometric data structure that can efficiently identify all the activated neurons in each training iteration and bypass the high-dimensional barrier of the prior approach.

This work raises a few interesting open questions for future study:
\begin{itemize}
    \item Can we apply our data structure, together with an analysis of the sparsity in training over-parameterized neural networks \cite{syz21,szz22}, to speed up multi-layer neural network training? 
    \item Many empirical results (e.g., \cite{cmf+20,chen2020mongoose_no}) indicate that only \emph{approximately} identifying the activated neurons (i.e., neurons with top-$k$ inner products) in each iteration may still be enough to train a neural network. Can we provide more theoretical understanding for these approaches? 
    \item Can we generalize our data structure to efficiently find the neurons activated by more general activation functions, e.g., leaky ReLU, ELU, softmax, etc?
\end{itemize}

\paragraph{Acknowledgements}

The author would like to thank Lichen Zhang for his helpful discussions.

\ifdefined\isarxivversion
\bibliographystyle{alpha}
\bibliography{ref}
\else
\bibliography{ref}

\newcommand{\etalchar}[1]{$^{#1}$}
\begin{thebibliography}{ADBMS98}

\bibitem[ABH{\etalchar{+}}18]{abh16}
Amir Abboud, Arturs Backurs, Thomas~Dueholm Hansen, Virginia
  Vassilevska~Williams, and Or~Zamir.
\newblock Subtree isomorphism revisited.
\newblock {\em ACM Transactions on Algorithms (TALG)}, 14(3):1--23, 2018.

\bibitem[ABW15]{abv15}
Amir Abboud, Arturs Backurs, and Virginia~Vassilevska Williams.
\newblock Tight hardness results for lcs and other sequence similarity
  measures.
\newblock In {\em 2015 IEEE 56th Annual Symposium on Foundations of Computer
  Science}, pages 59--78. IEEE, 2015.

\bibitem[AC09]{ac09}
Peyman Afshani and Timothy~M Chan.
\newblock Optimal halfspace range reporting in three dimensions.
\newblock In {\em Proceedings of the twentieth annual ACM-SIAM symposium on
  Discrete algorithms}, pages 180--186. SIAM, 2009.

\bibitem[ADBMS98]{adms98}
Pankaj~K Agarwal, Mark De~Berg, Jiri Matousek, and Otfried Schwarzkopf.
\newblock Constructing levels in arrangements and higher order voronoi
  diagrams.
\newblock {\em SIAM journal on computing}, 27(3):654--667, 1998.

\bibitem[ADH{\etalchar{+}}19a]{adhlw19}
Sanjeev Arora, Simon Du, Wei Hu, Zhiyuan Li, and Ruosong Wang.
\newblock Fine-grained analysis of optimization and generalization for
  overparameterized two-layer neural networks.
\newblock In {\em International Conference on Machine Learning}, pages
  322--332, 2019.

\bibitem[ADH{\etalchar{+}}19b]{adhlsw19}
Sanjeev Arora, Simon~S Du, Wei Hu, Zhiyuan Li, Ruslan Salakhutdinov, and
  Ruosong Wang.
\newblock On exact computation with an infinitely wide neural net.
\newblock In {\em NeurIPS}, 2019.

\bibitem[AEM92]{aem92}
Pankaj~K Agarwal, David Eppstein, and Jiri Matousek.
\newblock Dynamic half-space reporting, geometric optimization, and minimum
  spanning trees.
\newblock In {\em Annual Symposium on Foundations of Computer Science (FOCS)},
  volume~33, pages 80--80, 1992.

\bibitem[AIL{\etalchar{+}}15]{ailrs15}
Alexandr Andoni, Piotr Indyk, TMM Laarhoven, Ilya Razenshteyn, and Ludwig
  Schmidt.
\newblock Practical and optimal lsh for angular distance.
\newblock In {\em Advances in Neural Information Processing Systems (NIPS)},
  pages 1225--1233. Curran Associates, 2015.

\bibitem[AIR18]{air18}
Alexandr Andoni, Piotr Indyk, and Ilya Razenshteyn.
\newblock Approximate nearest neighbor search in high dimensions.
\newblock {\em arXiv preprint arXiv:1806.09823}, 7, 2018.

\bibitem[AR15]{ar15}
Alexandr Andoni and Ilya Razenshteyn.
\newblock Optimal data-dependent hashing for approximate near neighbors.
\newblock In {\em Proceedings of the forty-seventh annual ACM symposium on
  Theory of computing (STOC)}, pages 793--801, 2015.

\bibitem[ARN17]{arm17}
Alexandr Andoni, Ilya Razenshteyn, and Negev~Shekel Nosatzki.
\newblock Lsh forest: Practical algorithms made theoretical.
\newblock In {\em Proceedings of the Twenty-Eighth Annual ACM-SIAM Symposium on
  Discrete Algorithms (SODA)}, pages 67--78. SIAM, 2017.

\bibitem[AWW14]{avw14}
Amir Abboud, Virginia~Vassilevska Williams, and Oren Weimann.
\newblock Consequences of faster alignment of sequences.
\newblock In {\em International Colloquium on Automata, Languages, and
  Programming}, pages 39--51. Springer, 2014.

\bibitem[AWY14]{awy14}
Amir Abboud, Ryan Williams, and Huacheng Yu.
\newblock More applications of the polynomial method to algorithm design.
\newblock In {\em Proceedings of the twenty-sixth annual ACM-SIAM symposium on
  Discrete algorithms}, pages 218--230. SIAM, 2014.

\bibitem[AZLS19a]{als19_dnn}
Zeyuan Allen-Zhu, Yuanzhi Li, and Zhao Song.
\newblock A convergence theory for deep learning via over-parameterization.
\newblock In {\em ICML}, 2019.

\bibitem[AZLS19b]{als19_rnn}
Zeyuan Allen-Zhu, Yuanzhi Li, and Zhao Song.
\newblock On the convergence rate of training recurrent neural networks.
\newblock In {\em NeurIPS}, 2019.

\bibitem[BBK{\etalchar{+}}16]{bbk16}
Kevin Buchin, Maike Buchin, Maximilian Konzack, Wolfgang Mulzer, and Andr{\'e}
  Schulz.
\newblock Fine-grained analysis of problems on curves.
\newblock {\em EuroCG, Lugano, Switzerland}, 2016.

\bibitem[Ber24]{b24}
Sergei Bernstein.
\newblock On a modification of chebyshev's inequality and of the error formula
  of laplace.
\newblock {\em Ann. Sci. Inst. Sav. Ukraine, Sect. Math}, 1(4):38--49, 1924.

\bibitem[BGL17]{bgl17}
Karl Bringmann, Allan Gr{\o}nlund, and Kasper~Green Larsen.
\newblock A dichotomy for regular expression membership testing.
\newblock In {\em 2017 IEEE 58th Annual Symposium on Foundations of Computer
  Science (FOCS)}, pages 307--318. IEEE, 2017.

\bibitem[BI15]{bi15}
Arturs Backurs and Piotr Indyk.
\newblock Edit distance cannot be computed in strongly subquadratic time
  (unless seth is false).
\newblock In {\em Proceedings of the forty-seventh annual ACM symposium on
  Theory of computing}, pages 51--58, 2015.

\bibitem[BI16]{bi16}
Arturs Backurs and Piotr Indyk.
\newblock Which regular expression patterns are hard to match?
\newblock In {\em 2016 IEEE 57th Annual Symposium on Foundations of Computer
  Science (FOCS)}, pages 457--466. IEEE, 2016.

\bibitem[BIW19]{biw19}
Arturs Backurs, Piotr Indyk, and Tal Wagner.
\newblock Space and time efficient kernel density estimation in high
  dimensions.
\newblock In {\em NeurIPS}, pages 15773--15782, 2019.

\bibitem[BK18]{bk18}
Karl Bringman and Marvin K{\"u}nnemann.
\newblock Multivariate fine-grained complexity of longest common subsequence.
\newblock In {\em Proceedings of the Twenty-Ninth Annual ACM-SIAM Symposium on
  Discrete Algorithms}, pages 1216--1235. SIAM, 2018.

\bibitem[BM16]{bm16}
Karl Bringmann and Wolfgang Mulzer.
\newblock Approximability of the discrete fr{\'e}chet distance.
\newblock {\em Journal of Computational Geometry}, 7(2):46--76, 2016.

\bibitem[BPSW21]{bpsw21}
Jan van~den Brand, Binghui Peng, Zhao Song, and Omri Weinstein.
\newblock Training (overparametrized) neural networks in near-linear time.
\newblock In {\em 12th Innovations in Theoretical Computer Science Conference
  (ITCS)}, 2021.

\bibitem[Bri14]{bri14}
Karl Bringmann.
\newblock Why walking the dog takes time: Frechet distance has no strongly
  subquadratic algorithms unless seth fails.
\newblock In {\em 2014 IEEE 55th Annual Symposium on Foundations of Computer
  Science}, pages 661--670. IEEE, 2014.

\bibitem[CG19]{cg19}
Yuan Cao and Quanquan Gu.
\newblock Generalization bounds of stochastic gradient descent for wide and
  deep neural networks.
\newblock In {\em NeurIPS}, pages 10835--10845, 2019.

\bibitem[CGH{\etalchar{+}}19]{cgh+19}
Tianle Cai, Ruiqi Gao, Jikai Hou, Siyu Chen, Dong Wang, Di~He, Zhihua Zhang,
  and Liwei Wang.
\newblock Gram-gauss-newton method: Learning overparameterized neural networks
  for regression problems.
\newblock {\em arXiv preprint arXiv:1905.11675}, 2019.

\bibitem[Cha00]{c00}
Timothy~M Chan.
\newblock Random sampling, halfspace range reporting, and construction of
  ($\leq k$)-levels in three dimensions.
\newblock {\em SIAM Journal on Computing}, 30(2):561--575, 2000.

\bibitem[Cha02]{c02}
Moses~S Charikar.
\newblock Similarity estimation techniques from rounding algorithms.
\newblock In {\em Proceedings of the thiry-fourth annual ACM symposium on
  Theory of computing (STOC)}, pages 380--388, 2002.

\bibitem[Cha12]{cha12}
Timothy~M Chan.
\newblock Optimal partition trees.
\newblock {\em Discrete \& Computational Geometry}, 47(4):661--690, 2012.

\bibitem[Cha19]{c19}
Timothy~M Chan.
\newblock Orthogonal range searching in moderate dimensions: kd trees and range
  trees strike back.
\newblock {\em Discrete \& Computational Geometry}, 61(4):899--922, 2019.

\bibitem[Che20]{c20}
Lijie Chen.
\newblock On the hardness of approximate and exact (bichromatic) maximum inner
  product.
\newblock {\em Theory OF Computing}, 16(4):1--50, 2020.

\bibitem[CIP09]{cip09}
Chris Calabro, Russell Impagliazzo, and Ramamohan Paturi.
\newblock The complexity of satisfiability of small depth circuits.
\newblock In {\em International Workshop on Parameterized and Exact
  Computation}, pages 75--85. Springer, 2009.

\bibitem[CLP{\etalchar{+}}20]{chen2020mongoose_no}
Beidi Chen, Zichang Liu, Binghui Peng, Zhaozhuo Xu, Jonathan~Lingjie Li, Tri
  Dao, Zhao Song, Anshumali Shrivastava, and Christopher Re.
\newblock Mongoose: A learnable lsh framework for efficient neural network
  training.
\newblock In {\em OpenReview. net. Retrieved from https://openreview.
  net/forum}, 2020.

\bibitem[CLP{\etalchar{+}}21]{clp+21}
Beidi Chen, Zichang Liu, Binghui Peng, Zhaozhuo Xu, Jonathan~Lingjie Li, Tri
  Dao, Zhao Song, Anshumali Shrivastava, and Christopher Re.
\newblock Mongoose: A learnable lsh framework for efficient neural network
  training.
\newblock In {\em ICLR oral}, 2021.

\bibitem[CLS19]{cls19}
Michael~B Cohen, Yin~Tat Lee, and Zhao Song.
\newblock Solving linear programs in the current matrix multiplication time.
\newblock In {\em STOC}, 2019.

\bibitem[CMF{\etalchar{+}}20]{cmf+20}
Beidi Chen, Tharun Medini, James Farwell, Sameh Gobriel, Charlie Tai, and
  Anshumali Shrivastava.
\newblock Slide: In defense of smart algorithms over hardware acceleration for
  large-scale deep learning systems.
\newblock In {\em In Proceedings of the 3rd Conference on Machine Learning and
  Systems (MLSys)}, 2020.

\bibitem[CT17]{c17}
Timothy~M Chan and Konstantinos Tsakalidis.
\newblock Dynamic orthogonal range searching on the ram, revisited.
\newblock {\em Leibniz International Proceedings in Informatics, LIPIcs},
  77:281--2813, 2017.

\bibitem[CW16]{cw16}
Timothy~M Chan and Ryan Williams.
\newblock Deterministic apsp, orthogonal vectors, and more: Quickly
  derandomizing razborov-smolensky.
\newblock In {\em Proceedings of the twenty-seventh annual ACM-SIAM symposium
  on Discrete algorithms}, pages 1246--1255. SIAM, 2016.

\bibitem[CW19]{cw19}
Lijie Chen and Ryan Williams.
\newblock An equivalence class for orthogonal vectors.
\newblock In {\em Proceedings of the Thirtieth Annual ACM-SIAM Symposium on
  Discrete Algorithms}, pages 21--40. SIAM, 2019.

\bibitem[CWX22]{cvx22}
Timothy~M Chan, Virginia~Vassilevska Williams, and Yinzhan Xu.
\newblock Hardness for triangle problems under even more believable hypotheses:
  Reductions from real apsp, real 3sum, and ov.
\newblock {\em arXiv preprint arXiv:2203.08356}, 2022.

\bibitem[DIIM04]{diim04}
Mayur Datar, Nicole Immorlica, Piotr Indyk, and Vahab~S Mirrokni.
\newblock Locality-sensitive hashing scheme based on p-stable distributions.
\newblock In {\em Proceedings of the twentieth annual symposium on
  Computational geometry (SoCG)}, pages 253--262, 2004.

\bibitem[DIRW20]{dirw20}
Yihe Dong, Piotr Indyk, Ilya Razenshteyn, and Tal Wagner.
\newblock Learning space partitions for nearest neighbor search.
\newblock In {\em ICLR}. arXiv preprint arXiv:1901.08544, 2020.

\bibitem[DLW22]{dlv22}
Mina Dalirrooyfard, Ray Li, and Virginia~Vassilevska Williams.
\newblock Hardness of approximate diameter: Now for undirected graphs.
\newblock In {\em 2021 IEEE 62nd Annual Symposium on Foundations of Computer
  Science (FOCS)}, pages 1021--1032. IEEE, 2022.

\bibitem[DZPS19]{dzps19}
Simon~S Du, Xiyu Zhai, Barnabas Poczos, and Aarti Singh.
\newblock Gradient descent provably optimizes over-parameterized neural
  networks.
\newblock In {\em ICLR}, 2019.

\bibitem[GIKW18]{gikw17}
Jiawei Gao, Russell Impagliazzo, Antonina Kolokolova, and Ryan Williams.
\newblock Completeness for first-order properties on sparse structures with
  algorithmic applications.
\newblock {\em ACM Transactions on Algorithms (TALG)}, 15(2):1--35, 2018.

\bibitem[HLSY21]{hlsy21}
Baihe Huang, Xiaoxiao Li, Zhao Song, and Xin Yang.
\newblock Fl-ntk: A neural tangent kernel-based framework for federated
  learning convergence analysis.
\newblock In {\em ICML}, 2021.

\bibitem[IM98]{im98}
Piotr Indyk and Rajeev Motwani.
\newblock Approximate nearest neighbors: towards removing the curse of
  dimensionality.
\newblock In {\em Proceedings of the thirtieth annual ACM symposium on Theory
  of computing (STOC)}, pages 604--613, 1998.

\bibitem[IP01]{ip01}
Russell Impagliazzo and Ramamohan Paturi.
\newblock On the complexity of k-sat.
\newblock {\em Journal of Computer and System Sciences}, 62(2):367--375, 2001.

\bibitem[JGH18]{jgh18}
Arthur Jacot, Franck Gabriel, and Cl{\'e}ment Hongler.
\newblock Neural tangent kernel: Convergence and generalization in neural
  networks.
\newblock In {\em Advances in neural information processing systems}, pages
  8571--8580, 2018.

\bibitem[JT20]{jt20}
Ziwei Ji and Matus Telgarsky.
\newblock Polylogarithmic width suffices for gradient descent to achieve
  arbitrarily small test error with shallow relu networks.
\newblock In {\em ICLR}, 2020.

\bibitem[KKL20]{kkl20}
Nikita Kitaev, {\L}ukasz Kaiser, and Anselm Levskaya.
\newblock Reformer: The efficient transformer.
\newblock {\em arXiv preprint arXiv:2001.04451}, 2020.

\bibitem[KM20]{km20}
CS~Karthik and Pasin Manurangsi.
\newblock On closest pair in euclidean metric: Monochromatic is as hard as
  bichromatic.
\newblock {\em Combinatorica}, 40(4):539--573, 2020.

\bibitem[KT18]{kt18}
Robert Krauthgamer and Ohad Trabelsi.
\newblock Conditional lower bounds for all-pairs max-flow.
\newblock {\em ACM Transactions on Algorithms (TALG)}, 14(4):1--15, 2018.

\bibitem[LL18]{ll18}
Yuanzhi Li and Yingyu Liang.
\newblock Learning overparameterized neural networks via stochastic gradient
  descent on structured data.
\newblock In {\em NeurIPS}, 2018.

\bibitem[LS01]{ls01}
W.V. Li and Q.-M. Shao.
\newblock Gaussian processes: Inequalities, small ball probabilities and
  applications.
\newblock In {\em Stochastic Processes: Theory and Methods}, volume~19 of {\em
  Handbook of Statistics}, pages 533--597. Elsevier, 2001.

\bibitem[LSS{\etalchar{+}}20]{lsswy20}
Jason~D Lee, Ruoqi Shen, Zhao Song, Mengdi Wang, and Zheng Yu.
\newblock Generalized leverage score sampling for neural networks.
\newblock In {\em NeurIPS}, 2020.

\bibitem[Mat92a]{m92}
Jiri Matousek.
\newblock Efficient partition trees.
\newblock {\em Discrete \& Computational Geometry}, 8(3):315--334, 1992.

\bibitem[Mat92b]{m92b}
Jiri Matousek.
\newblock Reporting points in halfspaces.
\newblock {\em Computational Geometry}, 2(3):169--186, 1992.

\bibitem[OS20]{os19}
Samet Oymak and Mahdi Soltanolkotabi.
\newblock Toward moderate overparameterization: Global convergence guarantees
  for training shallow neural networks.
\newblock {\em IEEE Journal on Selected Areas in Information Theory},
  1(1):84--105, 2020.

\bibitem[Raz17]{r17}
Ilya Razenshteyn.
\newblock {\em High-dimensional similarity search and sketching: algorithms and
  hardness}.
\newblock PhD thesis, Massachusetts Institute of Technology, 2017.

\bibitem[Rub18]{rub18}
Aviad Rubinstein.
\newblock Hardness of approximate nearest neighbor search.
\newblock In {\em Proceedings of the 50th annual ACM SIGACT symposium on theory
  of computing}, pages 1260--1268, 2018.

\bibitem[RVW13]{rv13}
Liam Roditty and Virginia Vassilevska~Williams.
\newblock Fast approximation algorithms for the diameter and radius of sparse
  graphs.
\newblock In {\em Proceedings of the forty-fifth annual ACM symposium on Theory
  of computing}, pages 515--524, 2013.

\bibitem[SL14]{sl14}
Anshumali Shrivastava and Ping Li.
\newblock Asymmetric lsh (alsh) for sublinear time maximum inner product search
  (mips).
\newblock {\em Advances in Neural Information Processing Systems (NIPS)}, pages
  2321--2329, 2014.

\bibitem[SL15a]{sl15_www}
Anshumali Shrivastava and Ping Li.
\newblock Asymmetric minwise hashing for indexing binary inner products and set
  containment.
\newblock In {\em Proceedings of the 24th international conference on world
  wide web (WWW)}, pages 981--991, 2015.

\bibitem[SL15b]{sl15_uai}
Anshumali Shrivastava and Ping Li.
\newblock Improved asymmetric locality sensitive hashing (alsh) for maximum
  inner product search (mips).
\newblock In {\em Proceedings of the Thirty-First Conference on Uncertainty in
  Artificial Intelligence (UAI)}, pages 812--821, 2015.

\bibitem[SY19]{sy19}
Zhao Song and Xin Yang.
\newblock Quadratic suffices for over-parametrization via matrix chernoff
  bound.
\newblock {\em arXiv preprint arXiv:1906.03593}, 2019.

\bibitem[SYZ21]{syz21}
Zhao Song, Shuo Yang, and Ruizhe Zhang.
\newblock Does preprocessing help training over-parameterized neural networks?
\newblock {\em Advances in Neural Information Processing Systems}, 34, 2021.

\bibitem[SZZ21]{szz22}
Zhao Song, Lichen Zhang, and Ruizhe Zhang.
\newblock Training multi-layer over-parametrized neural network in subquadratic
  time.
\newblock {\em arXiv preprint arXiv:2112.07628}, 2021.

\bibitem[TOG17]{tog17}
Csaba~D Toth, Joseph O'Rourke, and Jacob~E Goodman.
\newblock {\em Handbook of discrete and computational geometry}.
\newblock CRC press, 2017.

\bibitem[Wil05]{wil05}
Ryan Williams.
\newblock A new algorithm for optimal 2-constraint satisfaction and its
  implications.
\newblock {\em Theoretical Computer Science}, 348(2-3):357--365, 2005.

\bibitem[Wil18a]{wil18}
Ryan Williams.
\newblock On the difference between closest, furthest, and orthogonal pairs:
  Nearly-linear vs barely-subquadratic complexity.
\newblock In {\em Proceedings of the Twenty-Ninth Annual ACM-SIAM Symposium on
  Discrete Algorithms}, pages 1207--1215. SIAM, 2018.

\bibitem[Wil18b]{williams2018some}
Virginia~Vassilevska Williams.
\newblock On some fine-grained questions in algorithms and complexity.
\newblock In {\em Proceedings of the International Congress of Mathematicians:
  Rio de Janeiro 2018}, pages 3447--3487. World Scientific, 2018.

\bibitem[ZG19]{zg19}
Difan Zou and Quanquan Gu.
\newblock An improved analysis of training over-parameterized deep neural
  networks.
\newblock In {\em NeurIPS}, pages 2053--2062, 2019.

\bibitem[Zha22]{z22}
Lichen Zhang.
\newblock Speeding up optimizations via data structures: Faster search, sample
  and maintenance.
\newblock Master's thesis, Carnegie Mellon University, 2022.

\bibitem[ZMG19]{zmg19}
Guodong Zhang, James Martens, and Roger~B Grosse.
\newblock Fast convergence of natural gradient descent for over-parameterized
  neural networks.
\newblock In {\em Advances in Neural Information Processing Systems (NeurIPS)},
  2019.

\bibitem[ZPD{\etalchar{+}}20]{zpdlsa20}
Yi~Zhang, Orestis Plevrakis, Simon~S Du, Xingguo Li, Zhao Song, and Sanjeev
  Arora.
\newblock Over-parameterized adversarial training: An analysis overcoming the
  curse of dimensionality.
\newblock In {\em NeurIPS}. arXiv preprint arXiv:2002.06668, 2020.

\end{thebibliography}
 \bibliographystyle{plainnat}

\fi

\newpage
\onecolumn
\appendix
\section*{Appendix}

\paragraph{Roadmap.}

We restate our notation and provide additional tools about probability in section~\ref{app:preli}. Then we present the results about sparsity in section~\ref{app:sparsity}. In section~\ref{app:algorithm}, we demonstrate the idea of two different correlation trees, DTree and WTree and present the full version of training algorithms using our data structure. In section~\ref{app:data_structure}, we provide detailed implementation and analysis of running time for our data structure. Section~\ref{app:training_time} presents the proof of running time for training a $2\mathrm{NN}(m,b)$ using DTree and WTree. 

\section{Preliminary}\label{app:preli}

\subsection{Basic Notation}
For any positive integer $n$, we use $[n]$ to denote the set $\{1,2,\cdots,n\}$. We use $\E[X]$ to denote the expected value of a random variable $X$, and $\Pr[Y]$ to denote the probability of a random event $Y$. For a matrix $M$, we write $M^\top$ to denote the transpose of $M$. We use $x^\top y$ to denote the inner product between vectors $x$ and $y$. We use $I_d$ to denote d-dimensional identity matrix. We use $\mathcal{N}(\mu; \sigma^2)$ to denote
the Gaussian distribution with mean $\mu$ and variance $\sigma^2$.

\subsection{Upper bound on the movement of weights per iteration}

The following Claim is quite standard in the literature, we omitt the details.
\begin{claim}[Corollary 4.1 in \cite{dzps19}, Lemma 3.8 in \cite{sy19}]\label{clm:4.1}
Let $\err(i)$ be defined as Definition~\ref{def:err}. 
If $\forall i \in[t], \| \err(i) \|_2^2 \leq ( 1 - \eta \lambda / 2 )^i \cdot \| \err(0) \|_2^2$, then 
\begin{align*}
 \| W(t+1) - W_r(0) \|_{\infty,2} \leq  4 \lambda^{-1} m^{-1/2} \cdot \sqrt{n} \cdot \| \err(0) \|_2  := D.
\end{align*}
\end{claim}

This claim show a uniform bound on the movement of weights.

Next, we introduce the definition of error of prediction.

\begin{definition}[Error of prediction]\label{def:err}
For each $t\in \{0,1,\cdots,T\}$, we define $\mathrm{err}(t)\in\R^n$ to be the error of prediction $\mathrm{err}(t)=y-u(t)$, where $u(t):=f(W(t),a,X)\in\R^n$
\end{definition}

\subsection{Probabilities}

We introduce the classical Bernstein inequality here.

\begin{lemma}[Bernstein inequality \cite{b24}]\label{lem:bernstein}
Assume $Z_1, \cdots, Z_n$ are $n$ i.i.d. random variables. $\forall i \in [n]$, $\E[Z_i]=0$ and $|Z_i| \leq M$ almost surely. Let $Z = \sum_{i=1}^n Z_i$. Then,
\begin{align*}
\Pr \left[ Z > t \right] \leq \exp \left( - \frac{ t^2/2 }{ \sum_{j=1}^n \E[Z_j^2]  + M t /3 } \right), \forall t > 0.
\end{align*}
\end{lemma}

Next we show an inequality on shifted small ball with Gaussian distribution.

\begin{claim}[Theorem 3.1 in \cite{ls01}]\label{clm:gaussain_anti_shift}
Let $b>0$ and $r>0$. Then,
\begin{align*}
    \exp(-b^2/2)\Pr_{x\sim \N(0,1)}[|x|\leq r] \leq ~ \Pr_{x\sim \N(0,1)}[|x-b|\leq r] \leq ~ \Pr_{x\sim \N(0,1)}[|x|\leq r].
\end{align*}
\end{claim}

We state the anti-concentration inequality here. 
\begin{lemma}[Anti-concentration for Gaussian distribution]\label{lem:anti_gaussian}
Let $Z \sim {\N}(0,\sigma^2)$.
Then, for $t>0$,
\begin{align*}
    \Pr[|Z|\leq t]\leq \frac{2t}{\sqrt{2\pi}\sigma}.
\end{align*}
\end{lemma}
\section{Sparsity}\label{app:sparsity}

In this section, we start by restating the result about sparsity after initialization in Section~\ref{sec:preli}. Then we show how to bound the number of fire neuron per iteration in Section~\ref{app:sparsity_per_iteration}. 

\subsection{Sparsity after initialization}
The goal of this section is to prove the sparsity after initialization of neural network.

We start by defining "fire" set.

\begin{definition}[fire set, Definition 3.7 in \cite{syz21}]\label{def:fire_set_app}
Fix a query point $x\in\R^d$, let $\mathcal{S}_{\mathrm{fire}}(x)$ denote the set of neurons that are "fire", i.e.,
\begin{align*}
    \mathcal{S}_{\mathrm{fire}}(x):=\{i\in[m]:\langle x,w_i\rangle >b\}
\end{align*}

\end{definition}

Next, we introduce the fire set for each training iteration.

\begin{definition}[fire set per iteration, Definition 3.7 in \cite{syz21}]\label{def:fire_set_per_iter_app}
For each data point $x_i\in\R^d,i\in[n]$, weight $w_r\in\R^d,r\in[m]$ and each iteration $t\in\{0,1,\cdots,T\}$, we define
\begin{align*}
    S_{i,\mathrm{fire}}(t)&:={r\in[m]:\langle x_i,w_r(t) \rangle} \\
    \tilde{S}_{r,\mathrm{fire}}(t)&:={i\in[n]:\langle x_i,w_r(t) \rangle}
\end{align*}
Also, we define $k_{i,t}:=|S_{i,\mathrm{fire}}(t)|$ and $\tilde{k}_{r,t}:=|\tilde{S}_{r,\mathrm{fire}}(t)|$
\end{definition}

With above definitions, we can state the sparsity after initialization.

\begin{lemma}[Sparsity after initialization, formal version of Lemma~\ref{lem:sparsity}, Lemma 3.8 in \cite{syz21}]\label{lem:sparsity_formal}
Let $b>0$ be tunable parameter. If we use the $\Phi_b$ as the activation function, then after the initialzation, with probability at least $1-\exp(-\Omega(m\cdot \exp(-b^2/2)))$, it holds that for input data $x$, the number of activated neurons $k_x$ is at most $O(m\cdot \exp(-b^2/2))$, where $m$ is the total number of neurons.
\end{lemma}

\subsection{Bounding the number of fired neuron per iteration per level}\label{app:sparsity_per_iteration}

In this section, we will show that for $t=0,1,\dots,T,k=0,1,\cdots,\log m$, the number of fire neurons $k_{i,k,t}=|{\cal S}_{i,k,\mathrm{fire}}(t)|$ is small with high probability.

We define the set of neurons that are flipping at time $t$:
\begin{definition}[flip set, Definition C.8 in \cite{syz21} ]
For each $i \in [n]$, for each time $t\in [T]$ let ${\cal S}_{i,\mathrm{flip}}(t) \subset [m]$ denote the set of neurons that are never flipped during the entire training process,
\begin{align*}
    {\cal S}_{i,\mathrm{flip}}(t) := & ~ \{ r \in [m] :  
     ~ \sgn(\langle w_r(t), x_i \rangle - b ) \neq \sgn( \langle w_r(t-1), x_i \rangle - b )  \} . 
\end{align*}
\end{definition}

Over all the iterations of training algorithm, there are some neurons that never flip states. We provide a mathematical formulation of that set, 
\begin{definition}[noflip set, Definition C.9 in \cite{syz21}]
For each $i \in [n]$, let $S_{i} \subset [m]$ denote the set of neurons that are never flipped during the entire training process,
\begin{align}\label{eq:noflip_def_app}
    S_{i} := & ~ \{ r \in [m] :  \forall t \in [T] 
     ~ \sgn(\langle w_r(t), x_i \rangle - b) = \sgn( \langle w_r(0), x_i \rangle - b ) \} . 
\end{align}
\end{definition}

In Lemma~\ref{lem:sparsity}, 
we already show that $k_{i,0}=O(m\cdot \exp(-b^2/2))$ for all $i\in [n]$ with high probability. We can show that it also holds for $t>0$.

\begin{lemma}[Bounding the number of fired neuron per iteration, Lemma C.10 in \cite{syz21}]\label{lem:bound_fire_neurons_formal}
Let $b\geq 0$ be a parameter, and let $\sigma_b(x)=\max\{x, b\}$ be the activation function. For each $i \in [n], t\in[T]$, $k_{i,t}$ is the number of activated neurons at the $t$-th iteration.  For $0<t\leq T$, with probability at least $1-n \cdot \exp \left(-\Omega(m)\cdot \min\{R, \exp(-b^2/2)\} \right)$, 
$k_{i,t}$ is at most $O(m\exp(-b^2/2))$ for all $i\in [n]$.
\end{lemma}

\section{Algorithm}\label{app:algorithm}

In this section, we explain the procedures for two correlation tree data structure. We use the same setting as Section~\ref{sec:correlation_tree}

For the DTree data structure, it contains $n$ binary trees indexed by $n$ data points and supports the following operations:
\begin{itemize}
    \item {\bf Initialize} It takes data points $\{x_1, \cdots, x_n\} \subset \R^d$ and weights $\{w_1, \cdots, w_m\} \subset \R^d$ as input and compute inner products of all weight-data pairs $(w_r, x_i)$. It uses these inner products to create $n$ different trees. For the $i$-th tree based on data point $x_i$, it is constructed from $m$ leaf nodes $\langle w_r, x_i \rangle, r \in [m]$ and satisfies the property that the value of parent node is the maximum value of its child nodes.
    \item {\bf Update} It takes a new weight $z \in \R^d$ and an index $r \in [m]$ as input. For the $i$-th tree, it calculate the new inner product $\langle z, x_i \rangle$ and stores the value into the $r$-th leaf node. Then it compares the new value with its parent node. It replaces parent node with new value if it is larger and continue this comparing process. Otherwise it stops. Repeat the same operation for all $n$ trees. 
    \item {\bf Query} It takes a threshold $b \in \R_{\geq 0}$ and an index $i \in [n]$ as input. Starting from the root of the $i$-th tree, it checks if its value is greater than threshold $b$. If no, search ends. If yes, it treats the child nodes as the root of a new subtree and repeat this searching process until stop. Then it finds all indices $r \in [m]$ that satisfy $\{w_r: \sgn (\langle w_r, x_i \rangle - b) \geq 0\}$.
\end{itemize}

For the WTree data structure, it contains $m$ binary trees indexed by $m$ weights and supports the following operations:
\begin{itemize}
    \item {\bf Initialize} Similar to DTree, it takes data points $\{x_1, \cdots, x_n\} \subset \R^d$ and weights $\{w_1, \cdots, w_m\} \subset \R^d$ as input and compute inner products of all weight-data pairs $(w_r, x_i)$. It uses these inner products to create $nm$ different trees. For the $r$-th tree based on weight $w_i$, it is constructed from $n$ leaf nodes $\langle w_r, x_i \rangle, i \in [n]$ and satisfies the property that the value of parent node is the maximum value of its child nodes.
    \item {\bf Update} It takes a new weight $z \in \R^d$ and an index $r \in [m]$ as input. Then it re-constructs the $r$-th tree with weight $z$.
    \item {\bf Query} It takes a threshold $b \in \R_{\geq 0}$ and an index $r \in [m]$ as input. Starting from the root of the $r$-th tree, it checks if its value is greater than threshold $b$. If no, search ends. If yes, it treats the child nodes as the root of a new subtree and repeat this searching process until stop. Then it finds all indices $i \in [n]$ that satisfy $\{w_r: \sgn (\langle w_r, x_i \rangle - b) \geq 0\}$.
\end{itemize}

\begin{algorithm}[!ht]
\caption{Correlation DTree Data Structure}
\small
\begin{algorithmic}[1]
    \algrenewcommand\algorithmicprocedure{\textbf{data structure}}
    \Procedure{CorrelationDTree}{}
        \State {\bf procedures:}
        \State \hspace{4mm} \textsc{Init}($S\subset{\R^d}, W\subset{\R^d}, n, m, d$) \Comment{Initialize the data structure via building $n$ trees}
        \State \hspace{4mm} \textsc{Query}($i,b$)\Comment{$i\in[n],b\in \R$. Output the set $\{r \in [m]: \sgn(\langle w_r, x_i\rangle -b)\geq 0\}$}
        \State \hspace{4mm} \textsc{Update}($x,i$)\Comment{Update the $i$'th point  in $\R^d$ with $x$}
    \EndProcedure
\end{algorithmic}
\end{algorithm}

\begin{algorithm}[!ht]
\caption{Correlation WTree Data Structure}
\small
\begin{algorithmic}[1]
    \algrenewcommand\algorithmicprocedure{\textbf{data structure}}
    \Procedure{CorrelationWTree}{}
        \State {\bf procedures:}
        \State \hspace{4mm} \textsc{Init}($S\subset{\R^d}, W\subset{\R^d}, n, m, d$) \Comment{Initialize the data structure via building $m$ trees}
        \State \hspace{4mm} \textsc{Query}($r,b$)\Comment{$r\in[m],b\in \R$. Output the set $\{i \in [n]: \sgn(\langle w_r, x_i\rangle -b)\geq 0\}$}
        \State \hspace{4mm} \textsc{Update}($w,r$)\Comment{Update the $r$'th point in $\R^d$ with $w$}
    \EndProcedure
\end{algorithmic}
\end{algorithm}

We present \textsc{MakeMaxTree} algorithm (Algorithm~\ref{alg:make_tree_app}) which shows how to construct a tree satisfying the property that the value of parent node is the max value of its child node.

\begin{algorithm}[!ht]\caption{Make MaxTree}\label{alg:make_tree_app}
\begin{algorithmic}[1]
\Procedure{MakeMaxTreeInner}{$r_1,\cdots,r_n$} %
    \If{$n = 1$}
        \State \Return $r_1$
    \Else
        \For{$i \in [n/2]$} 
            \State Create node $r'_{i}$
            \If{$r_{2i-1}$.value $>$ $r_{2i}$.value}
                \State $r'_{i} \gets r_{2i-1}$
            \Else
                \State $r'_{i} \gets r_{2i}$
            \EndIf
            \State Insert $r_{2i-1}$ as left child
            \State Insert $r_{2i}$ as right child
        \EndFor
        \State \Comment{If $n$ is odd, create a parent node for the last node.}
        \State \Return \textsc{MakeMaxTreeInner}($\{r'_1, \cdots, r'_i\}$)
    \EndIf
\EndProcedure 

\Procedure{MakeMaxTree}{$u_1,\cdots,u_n$} %
    \For{$i \in [n]$}
        \State Create nodes $r_i$
        \State $r_i$.value $\gets u_i$
    \EndFor
    \State \Return \textsc{MakeMaxTreeInner}($r_1,\cdots,r_n$)
\EndProcedure 
\end{algorithmic}
\end{algorithm}

We then give two training algorithms (Algorithm~\ref{alg:nn_dtree_app} and Algorithm~\ref{alg:nn_wtree_app}) to show how DTree and WTree help in training neural network efficiently.

\begin{algorithm}[!ht] 
\caption{Training Neural Network via building $n$ trees, where each tree is the correlation between one data point and all the weights}\label{alg:nn_dtree_app}
\small
\begin{algorithmic}[1]
	\Procedure{TrainingWithPreprocessWeights}{$\{x_i\}_{i\in [n]}, \{y_i\}_{i\in [n]}$,$n$,$m$,$d$} \Comment{Theorem~\ref{thm:running_time_data}}
	\State \blue{/*Initialization step*/}
	\State Sample $W(0)$ and $a$ according to Definition~
    \State $b\gets \sqrt{0.4\log m}$.
    \State \blue{/*A dynamic data-structure*/}
    \State \textsc{CorrelationDTree} \textsc{cdt} \Comment{Theorem~\ref{thm:correlation_tree_data_structure}}
    \State \textsc{cdt}.\textsc{Init}($\{ x_{i} \}_{ i \in [n] }, \{w_r(0)\}_{r\in [m]}, n, m, d$) \Comment{It takes $\Tinit(n,m,d)$ time. Alg.~\ref{alg:correlation_tree_init_app}}
    \State \blue{/*Iterative step*/}
	\For{$t=0 \to T$}
	    \State \blue{/*Forward computation step*/}
	    \For{$i=1 \to n$} \label{ln:classical_1}
	        \State $S_{i,\mathrm{fire}}\gets \textsc{cdt}.\textsc{Query}(i,b)$ \Comment{It takes $\Tquery(m,k_{i,t})$ time. Alg.~\ref{alg:correlation_tree_query_data}}
		    \State $u(t)_i \leftarrow \frac{1}{ \sqrt{m} } \sum_{r\in \mathcal{S}_{i,\mathrm{fire}}} a_r \cdot \sigma_b(w_r(t)^\top x_i)$ \Comment{It takes $O(d\cdot k_{i,t})$ time}
		\EndFor
		\State \blue{/*Backward computation step*/}
		\State $P \leftarrow 0^{n \times m}$ \Comment{$P \in \R^{n \times m}$}
		\For{$i = 1 \to n$} \label{ln:classical_2}
		    \For{$r \in {\cal S}_{i,\mathrm{fire}}$}
			    \State $P_{i,r} \leftarrow \frac{1}{\sqrt{m}} a_r \cdot \sigma_b'( w_r(t)^\top x_i )$ 
			\EndFor
		\EndFor
		\State $M\gets X\diag(y-u(t))$ \Comment{$M\in \R^{d\times n}$, it takes $O(n\cdot d)$ time}
		\State $\Delta W \leftarrow \underbrace{M}_{d\times n}\underbrace{P}_{n\times m}$\label{ln:compute_delta_1}\Comment{$\Delta W\in \R^{d\times m}$, it takes $O(d\cdot \nnz(P))$ time, $\nnz(P) = O(nm^{4/5})$}
		\State $W(t+1)\gets W(t)-\eta \cdot \Delta W$. 
		\State \blue{/*Update data structure*/}
		\State Let $Q \subset [m]$ where for each $r \in Q$, the $\Delta W_{*,r}$ is not all zeros \Comment{$|Q| \leq O(n m^{4/5})$}
		\For{$r \in Q$}
		    \State \textsc{cdt}.\textsc{Update}$(w_r(t+1),r)$ \label{lin:weight_update} \Comment{Alg.~\ref{alg:correlation_tree_update_weight}} 
		\EndFor
	\EndFor
	\State \Return $W$ \Comment{$W \in \R^{d \times m}$}
	\EndProcedure
\end{algorithmic}
\end{algorithm}

\begin{algorithm}[!ht]
\caption{Training Neural Network via building $m$ trees, where each tree is the correlation between one weight and all the data points}\label{alg:nn_wtree_app}
\small
\begin{algorithmic}[1]
	\Procedure{TrainingWithProcessData}{$\{x_i\}_{i\in[n]},\{y_i\}_{i\in [n]}$,$n$,$m$,$d$} \Comment{Theorem~\ref{thm:running_time_weight}}
	\State \blue{/*Initialization step*/}
	\State Sample $W(0)$ and $a$ according to Definition~
    \State $b\gets \sqrt{0.4\log m}$.
   \State \blue{/*A static data-structure*/}
    \State \textsc{CorrelationWTree} \textsc{cwt} \Comment{Algorithm~\ref{alg:correlation_tree_init_app}, Part 2 of Theorem~\ref{thm:correlation_tree_data_structure}}
    \State \textsc{cwt}.\textsc{Init}($\{ x_{i} \}_{ i \in [n] }, \{w_r(0)\}_{r\in [m]}, n, m, d$) \Comment{It takes $\Tinit(n,m,d)$ time}
    \State \blue{/*Initialize $\wt{S}_{r,\mathrm{fire}}$ and $S_{i,\mathrm{fire}}$ */}
    \State \Comment{It takes $\sum_{r=1}^{m}\Tquery(n,\wt{k}_{r,t}) = O(m^{4/5}n\cdot\log n)$ time} 
     \State  $\wt{S}_{r,\mathrm{fire}} \leftarrow \emptyset$ for $r \in [m]$. \Comment{$\wt{S}_{r,\mathrm{fire}}$ is the set of samples, for which neuron $r$ fires}\label{lin:init_b}
    \State $S_{i,\mathrm{fire}} \leftarrow \emptyset$ for $i \in [n]$. \Comment{$S_{i,\mathrm{fire}}$ is the set of neurons, which fire for $x_i$}
    \For{$r = 1 \to m$} 
        
        \State $\wt{S}_{r,\mathrm{fire}} \leftarrow \textsc{cwt}.\textsc{Query}(r,b)$ 
        \For{$i \in \wt{S}_{r,\mathrm{fire}}$}
            \State $S_{i, \mathrm{fire}}.\textsc{Add}(r)$
        \EndFor
    \EndFor\label{lin:init_e}
    
    \State \blue{/*Iterative step*/}
	\For{$t=1 \to T$}
	    \State \blue{/*Forward computation step*/}
	    \For{$i=1 \to n$}\label{lin:forward_b}
		    \State $u(t)_i \leftarrow \frac{1}{ \sqrt{m} } \sum_{r\in \mathcal{S}_{i,\mathrm{fire}}} a_r \cdot \sigma_b(w_r(t)^\top x_i)$ \Comment{It takes $O(d\cdot k_{i,t})$ time}
		\EndFor
		\State \blue{/*Backward computation step*/}
		\State $P \leftarrow 0^{n \times m}$ \Comment{$P \in \R^{n \times m}$}
		\For{$i = 1 \to n$} \label{ln:classical_3}
		    \For{$r \in {\cal S}_{i,\mathrm{fire}}$}
			    \State $P_{i,r} \leftarrow \frac{1}{\sqrt{m}} a_r \cdot \sigma_b'( w_r(t)^\top x_i )$ 
			\EndFor
		\EndFor
		\State $M\gets X\diag(y-u(t))$ \Comment{$M\in \R^{d\times n}$, it takes $O(n\cdot d)$ time}
		\State $\Delta W \leftarrow \underbrace{M}_{d\times n}\underbrace{P}_{n\times m}$\label{ln:compute_delta_2}\Comment{$\Delta W\in \R^{d\times m}$, it takes $O(d\cdot \nnz(P))$ time, $\nnz(P) = O(nm^{4/5})$}
		\State $W(t+1)\gets W(t)-\eta \cdot \Delta W$. \label{lin:backward_e}
		\State \blue{/*Update $\wt{S}_{r,\mathrm{fire}}$ and $S_{i,\mathrm{fire}}$ step*/}
		\State \Comment{It takes $O(\sum_{i=1}^n k_{i, t} + \sum_{r \in S_{[n], \mathrm{fire}}}\Tquery(n, d, \wt{k}_{r, t+ 1}))=O(n \cdot \log n \cdot m^{4/5})$ }
		\State $S_{[n], \mathrm{fire}} \leftarrow \cup_{i \in [n]}\mathcal{S}_{i, \mathrm{fire}}$ \label{lin:data_update_b}
		\For{$r \in S_{[n], \mathrm{fire}}$} 
		\For{$i \in \wt{S}_{r,\mathrm{fire}}$}  \Comment{Removing old fired neuron indices. It takes $O(\wt{k}_{r, t})$ time} %
            \State $S_{i, \mathrm{fire}}.\textsc{Del}(r)$ 
        \EndFor
        \State $\textsc{cwt}.\textsc{Update}(w_r(t+1),r)$  \Comment{It takes $\Tupdate(n,d)$ time}
        \State $\wt{S}_{r,\mathrm{fire}} \leftarrow \textsc{cwt}.\textsc{Query}(r,b)$ \Comment{It takes $\Tquery(n,d, \wt{k}_{r,t+1})$ time}
        \For{$i \in \wt{S}_{r,\mathrm{fire}}$}  \Comment{Adding new fired neuron indices. It takes $O(\wt{k}_{r, t+1})$ time}
            \State $S_{i, \mathrm{fire}}.\textsc{Add}(r)$
        \EndFor\label{lin:data_update_e}
    \EndFor
	\EndFor
	\State \Return $W$ \Comment{$W \in \R^{d \times m}$}
	\EndProcedure
\end{algorithmic}
\end{algorithm}

\section{Correlation Tree Data Structure}\label{app:data_structure}
In this section, we demonstrate detailed results for DTree and WTree data structures. 

\subsection{Correlation DTree data structure}\label{app:dtree}
We start by stating the main theorem of correlation DTree data structure.
\begin{theorem}[Correlation DTree data structure]\label{thm:correlation_tree_data_structure_app}
There exists a data structure with the following procedures:
\begin{itemize}
    \item \textsc{Init}$(\{w_1,w_2, \cdots, w_m\} \subset \R^d, \{x_1, x_2, \cdots, x_n\} \subset \R^d,n\in\mathbb{N},m\in\mathbb{N},d\in\mathbb{N})$. Given a series of weights $w_1,w_2,\cdots,w_m$ and datas $x_1, x_2, \cdots, x_n$ in d-dimensional space, it preprocesses in time $O(nmd)$
    \item \textsc{Update}$(z\in\R^d,r\in [m])$. Given a weight $z$ and index $r$, it updates weight $w_r$ with $z$ in time $O(n\cdot(d+\log m))$ 
    \item \textsc{Query}$(i \in [n],\tau \in \R)$. Given an index $i$ indicating data point $x_i$ and a threshold $\tau$, it finds all index $r\in[m]$ such that $\langle w_r,x_i \rangle >\tau$ in time $O(|\tilde{S}(\tau)|\cdot \log m)$, where $\tilde{S}(\tau):=\{r:\langle w_r,x_i \rangle >\tau\}$  
\end{itemize}
\end{theorem}

\begin{algorithm}[!ht]\caption{Correlation DTree data structure} \label{alg:correlation_tree_init_app} 
\begin{algorithmic}[1]
\State {\bf data structure} \textsc{CorrelationDTree} \Comment{Theorem~\ref{thm:correlation_tree_data_structure_app}}%
\State {\bf members}
\State \hspace{4mm} $W \in \R^{m\times d}$ ($m$ weight vectors )
\State \hspace{4mm} $X \in \R^{n \times d}$ ($n$ data points)
\State \hspace{4mm} Binary tree $T_1, T_2, \cdots, T_n$ \Comment{$n$ binary search trees}
\State {\bf end members}
\State
\State {\bf public:}
\Procedure{Init}{$w_1,w_2, \cdots, w_m \in \R^d, m, x_1, x_2, \cdots, x_n \in \R^d$, $n$, $m$,  $d$}  \Comment{Lemma~\ref{lem:correlation_tree_init_formal}}
    \For{$i=1 \to n$} \label{lin:init_first_loop_app}
        \State $x_i \gets x_i$
    \EndFor
    \For{$j=1 \to m$} \label{lin:init_second_loop_app}
        \State $w_j \gets w_j$
    \EndFor
    \For{$i=1 \to n$} \Comment{for data point, we create a tree} \label{lin:init_outer_loop_app}
        \For{$j=1 \to m$} \label{lin:init_inner_loop_app}
            \State $u_j \gets \langle x_i, w_j \rangle$ \label{lin:init_inner_product_app}
        \EndFor
        \State $T_i \gets \textsc{MakeTree}(u_1, \cdots, u_m)$\label{lin:init_make_binary_tree_app} \Comment{Each node stores the maximum value for his two children}
    \EndFor
\EndProcedure
\State {\bf end data structure}
\end{algorithmic}
\end{algorithm}

\begin{algorithm}[!ht]\caption{Correlation DTrees}\label{alg:correlation_tree_update_weight}
\begin{algorithmic}[1]
\State {\bf data structure} \textsc{CorrelationTree} \Comment{Theorem~\ref{thm:correlation_tree_data_structure_app}}
\State {\bf public:}
\Procedure{Update}{$z\in\R^d, r \in [m]$} \Comment{Lemma~\ref{lem:correlation_tree_update_weight_formal}} 
\State $w_r \gets z$ 
\For{$i=1 \to n$} \label{lin:update_loop_app}
    \State $l \gets$ the $l$-th leaf of tree $T_i$ \label{lin:update_find_leaf_app}
    \State $l.\text{value} = \langle z, x_i \rangle$ \label{lin:update_inner_product_app}
    \While{$l$ is not root}
        \State $p$ $\gets$ parent of $l$
        \State $p.\text{value} \gets \max \{ p.\text{value}, l.\text{value} \}$
        \State $l \gets p$
    \EndWhile
\EndFor
\EndProcedure
\State {\bf end data structure}
\end{algorithmic}
\end{algorithm}

\begin{algorithm}[!ht]\caption{Correlation DTrees }\label{alg:correlation_tree_query_data}
\begin{algorithmic}[1]
\State {\bf data structure} \textsc{CorrelationDTree} \Comment{Theorem~\ref{thm:correlation_tree_data_structure_app}} 
\State {\bf public:}
\Procedure{Query}{$i \in [n], \tau \in \R_{\geq 0}$} \Comment{Lemma~\ref{lem:correlation_tree_query_data_formal}} 
\State \Return \textsc{Find}($\tau,\mathrm{root}(T_i)$)
\EndProcedure
\State  
\State {\bf private:}  
\Procedure{Find}{$\tau \in \R_{\geq 0}, r\in T$}
\If{$r$ is leaf}
\State \Return $r$
\Else
\State $r_1\gets$ left child of $r$, $r_2\gets$ right child of $r$
\If{$r_1.\text{value} \geq \tau$}
    \State $S_1 \gets $\textsc{Find}$(\tau,r_1)$
\EndIf
\If{$r_2.\text{value} \geq \tau$}
    \State $S_2 \gets $\textsc{Find}$(\tau,r_2)$
\EndIf
\EndIf
\State \Return $S_1 \cup S_2$
\EndProcedure
\State {\bf end data structure}
\end{algorithmic}
\end{algorithm}

\subsection{Running time for \textsc{CorrelationDTree}}\label{app:dtree_runningtime}
The goal of this secion is to prove the running time of \textsc{Init}, \textsc{Update} and \textsc{Query}.

We start by showing the running time of \textsc{Init}.

\begin{lemma}[Running time of \textsc{Init}]\label{lem:correlation_tree_init_formal}
Given a series of weights $\{w_1,w_2,\cdots,w_m\}\subset\R^d$ and datas $\{x_1, x_2, \cdots, x_n\}\subset\R^d$, it preprocesses in time $O(nmd)$
\end{lemma}
\begin{proof}
The \textsc{Init} consists of two independent forloop and two recursive forloops. The first forloop (start from line~\ref{lin:init_first_loop}) has $n$ interations, which takes $O(n)$ time. The second forloop (start from line~\ref{lin:init_second_loop}) has $m$ iterations, which takes $O(m)$ time.
Now we consider the recursive forloop. The outer loop (line~\ref{lin:init_outer_loop}) has $n$ iterations. In inner loop has $m$ iterations. In each iteration of the inner loop, line~\ref{lin:init_inner_product} takes $O(d)$ time. Line~\ref{lin:init_make_binary_tree} takes $O(m)$ time.
Putting it all together, the running time of \textsc{Init} is
\begin{align*}
     & ~ O(n+m+n(md+m))\\
    =& ~ O(nmd)
\end{align*}
Thus, we complete the proof.
\end{proof}

Next, we analyze the running time of \textsc{Update}.

\begin{lemma}[Running time of \textsc{Update}]\label{lem:correlation_tree_update_weight_formal}
Given a weight $z\in\R^d$ and index $j\in [m]$, it updates weight $w_j$ with $z$ in time $O(n\cdot(d+\log m))$
\end{lemma}
\begin{proof}
The running time of \textsc{Update} mainly comes from the forloop (line~\ref{lin:update_loop}), which consists of $n$ iterations. In each iteration, line ~\ref{lin:update_find_leaf} takes $O(\log m)$ time, line~\ref{lin:update_inner_product} takes $O(d)$ time and the while loop takes $O(\log m)$ time since it go through a path bottom up.
Putting it together, the running time of \textsc{Update} is $O(n(d+\log m))$.
\end{proof}

Finally, we state the running time for \textsc{Query} procedure.
\begin{lemma}[Running time of \textsc{Query}]\label{lem:correlation_tree_query_data_formal}
Given a query $q\in \R^d$ and a threshold $\tau > 0$, it finds all index $i\in[n]$ such that $\langle w_i,q \rangle >\tau$ in time $O(|S(\tau)|\cdot \log m)$, where $S(\tau) := \{i:\langle w_i,q \rangle >\tau\}$
\end{lemma}
\begin{proof}
The running time comes from \textsc{Find} with input $\tau$ and $\mathrm{root}(T_i)$. In \textsc{Find}, we start from the root node $r$ and find indices in a recursive way. The \textsc{Init} guarantees that for a node $r$ satisfying $\mathrm{r.value}>\tau$, the sub-tree with root $r$ must contains a leaf whose value is greater than $\tau$ If not satisfied, all the values of the nodes in the sub-tree with root $r$ is less than$\tau$. This guarantees that all the paths it search does not have any branches that leads to the leaf we don't want and it will report all the indices $i$ satisfying $\langle w_i,q\rangle>0$. Note that the depth of $T$ is $O(\log n)$, the running time of \textsc{Query} is $O(|S(\tau)|\cdot\log n)$
\end{proof}

\subsection{Correlation WTree data structure}\label{app:wtree}
In this section, we state the main theorem of correlation wtree data structure.
\begin{theorem}[Correlation WTree data structure]\label{thm:correlation_wtree_data_structure_app}
There exists a data structure with the following procedures:
\begin{itemize}
    \item \textsc{Init}$(\{w_1,w_2, \cdots, w_m\} \subset \R^d, \{x_1, x_2, \cdots, x_n\} \subset \R^d,n\in\mathbb{N},m\in\mathbb{N},d\in\mathbb{N})$. Given a series of weights $w_1,w_2,\cdots,w_m$ and datas $x_1, x_2, \cdots, x_n$ in d-dimensional space, it preprocesses in time $O(nmd)$
    \item \textsc{Update}$(z\in\R^d,r\in [m])$. Given a weight $z$ and index $r$, it updates weight $w_r$ with $z$ in time $O(nd)$ 
    \item \textsc{Query}$(r \in [m],\tau \in \R)$. Given an index $r$ indicating weight $w_r$ and a threshold $\tau$, it finds all index $i\in[n]$ such that $\langle w_r,x_i \rangle >\tau$ in time $O(|S(\tau)|\cdot \log m)$, where $S(\tau):=\{i:\langle w_r,x_i \rangle >\tau\}$ 
\end{itemize}
\end{theorem}

\begin{algorithm}[!ht]\caption{Correlation WTree data structure} \label{alg:correlation_wtree_init} 
\begin{algorithmic}[1]
\State {\bf data structure} \textsc{CorrelationWTree} \Comment{Theorem~\ref{thm:correlation_wtree_data_structure_app}}
\State {\bf members}
\State \hspace{4mm} $W \in \R^{m\times d}$ ($m$ weight vectors )
\State \hspace{4mm} $X \in \R^{n \times d}$ ($n$ data points)
\State \hspace{4mm} Binary tree $T_1, T_2, \cdots, T_M$ \Comment{$m$ binary search trees}
\State {\bf end members}
\State
\State {\bf public:}
\Procedure{Init}{$w_1,w_2, \cdots, w_m \in \R^d, m, x_1, x_2, \cdots, x_n \in \R^d$, $n$, $m$,  $d$}  \Comment{Lemma~\ref{lem:correlation_wtree_init_formal}}
    \For{$i=1 \to n$} \label{lin:initw_first_loop}
        \State $x_i \gets x_i$
    \EndFor
    \For{$j=1 \to m$} \label{lin:initw_second_loop}
        \State $w_j \gets w_j$
    \EndFor
    \For{$i=1 \to m$} \Comment{for weight, we create a tree} \label{lin:initw_outer_loop}
        \For{$j=1 \to n$} \label{lin:initw_inner_loop}
            \State $u_j \gets \langle x_i, w_j \rangle$ \label{lin:initw_inner_product}
        \EndFor
        \State $T_i \gets \textsc{MakeTree}(u_1, \cdots, u_n)$\label{lin:initw_make_binary_tree} \Comment{Each node stores the maximum value for his two children}
    \EndFor
\EndProcedure
\State {\bf end data structure}
\end{algorithmic}
\end{algorithm}

\begin{algorithm}[!ht]\caption{Correlation WTrees}\label{alg:correlation_wtree_update_weight}
\begin{algorithmic}[1]
\State {\bf data structure} \textsc{CorrelationWTree} \Comment{Theorem~\ref{thm:correlation_wtree_data_structure_app}}
\State {\bf public:}
\Procedure{Update}{$z\in\R^d, r \in [m]$} \Comment{Lemma~\ref{lem:correlation_wtree_update_weight_formal}} 
\State $w_r \gets z$ 
\For{$j=1 \to n$}
    \State $u_j \gets \langle x_j, w_r \rangle$
    \State $T_i \gets \textsc{MakeTree}(u_1, \cdots, u_n)$ \Comment{Each node stores the maximum value for his two children}
\EndFor
\EndProcedure
\State {\bf end data structure}
\end{algorithmic}
\end{algorithm}

\begin{algorithm}[!ht]\caption{Correlation WTree }\label{alg:correlation_wtree_query_data}
\begin{algorithmic}[1]
\State {\bf data structure} \textsc{CorrelationWTree}
\State {\bf public:}
\Procedure{Query}{$r \in [m], \tau \in \R_{\geq 0}$} \Comment{Lemma~\ref{lem:correlation_wtree_query_data_formal}}  
\State \Return \textsc{Find}($\tau,\mathrm{root}(T_r)$)
\EndProcedure
\State  
\State {\bf private:}  
\Procedure{Find}{$\tau \in \R_{\geq 0}, r\in T$}
\If{$r$ is leaf}
\State \Return $r$
\Else
\State $r_1\gets$ left child of $r$, $r_2\gets$ right child of $r$
\If{$r_1.\text{value} \geq \tau$}
    \State $S_1 \gets $\textsc{Find}$(\tau,r_1)$
\EndIf
\If{$r_2.\text{value} \geq \tau$}
    \State $S_2 \gets $\textsc{Find}$(\tau,r_2)$
\EndIf
\EndIf
\State \Return $S_1 \cup S_2$
\EndProcedure
\State {\bf end data structure}
\end{algorithmic}
\end{algorithm}

\subsection{Running time for Correlation WTree}\label{app:wtree_runningtime}
The goal of this secion is to prove the running time of \textsc{Init}, \textsc{Update} and \textsc{Query}.

As in DTree, we first show the running time for \textsc{Init}.

\begin{lemma}[Running time of \textsc{Init}]\label{lem:correlation_wtree_init_formal}
Given a series of weights $\{w_1,w_2,\cdots,w_m\}\subset\R^d$ and datas $\{x_1, x_2, \cdots, x_n\}\subset\R^d$, it preprocesses in time $O(nmd)$
\end{lemma}
\begin{proof}
The \textsc{Init} consists of two independent forloop and two recursive forloops. The first forloop (start from line~\ref{lin:initw_first_loop}) has $n$ interations, which takes $O(n)$ time. The second forloop (start from line~\ref{lin:initw_second_loop}) has $m$ iterations, which takes $O(m)$ time.
Now we consider the recursive forloop. The outer loop (line~\ref{lin:initw_outer_loop}) has $m$ iterations. In inner loop has $n$ iterations. In each iteration of the inner loop, line~\ref{lin:initw_inner_product} takes $O(d)$ time. Line~\ref{lin:initw_make_binary_tree} takes $O(n)$ time.
Putting it all together, the running time of \textsc{Init} is
\begin{align*}
     & ~ O(n+m+m(nd+n))\\
    =& ~ O(nmd)
\end{align*}
Thus, we complete the proof.
\end{proof}

Next, we turn to the running time for \textsc{Update}.

\begin{lemma}[Running time of \textsc{Update}]\label{lem:correlation_wtree_update_weight_formal}
Given a weight $z\in\R^d$ and index $r\in [m]$, it updates weight $w_j$ with $z$ in time $O(nd)$
\end{lemma}
\begin{proof}
In this procedure, it generates a new tree for weight $w_r$ with $n$ leaves, which takes $O(nd)$ time. Thus, we complete the proof.
\end{proof}

Finally, we present the running time of \textsc{Query}.

\begin{lemma}[Running time of \textsc{Query}]\label{lem:correlation_wtree_query_data_formal}
Given a query $q\in \R^d$ and a threshold $\tau > 0$, it finds all index $i\in[n]$ such that $\langle w_i,q \rangle >\tau$ in time $O(|S(\tau)|\cdot \log m)$, where $S(\tau) := \{i:\langle w_i,q \rangle >\tau\}$
\end{lemma}

\begin{proof}
The running time comes from \textsc{Find} with input $\tau$ and $\mathrm{root}(T_i)$. In \textsc{Find}, we start from the root node $r$ and find indices in a recursive way. The \textsc{Init} guarantees that for a node $r$ satisfying $\mathrm{r.value}>\tau$, the sub-tree with root $r$ must contains a leaf whose value is greater than $\tau$ If not satisfied, all the values of the nodes in the sub-tree with root $r$ is less than$\tau$. This guarantees that all the paths it search does not have any branches that leads to the leaf we don't want and it will report all the indiex $i$ satisfying $\langle w_i,q\rangle>0$. Note that the depth of $T$ is $O(\log n)$, the running time of \textsc{Query} is $O(|S(\tau)|\cdot\log n)$
\end{proof}

\section{More Details of Our Training Algorithms}\label{app:training_time}

\subsection{Weights Preprocessing}

In this section, we present the formal version of our training algorithm using DTree, which preprocessing weights for each data point.

\begin{theorem}[Running time part, formal version of Theorem~\ref{thm:running_time_data}]\label{thm:running_time_data_formal}
Given $n$ data points in $\R^d$. Running gradient descent algorithm (Algorithm~\ref{alg:nn_dtree_app}) on $\mathrm{2NN}(m,b=\sqrt{0.4\log m})$ (Definition~\ref{def:neural_network}) 
the expected cost per-iteration of the gradient descent algorithm is
\begin{align*}
    O(m^{4/5}n^2d)
\end{align*}
\end{theorem}
\begin{proof}
The per-step time complexity is
\begin{align*}
{\cal T} = & ~ {\cal T}_1 + {\cal T}_2 + {\cal T}_3 \\
   = & ~ \sum_{i=1}^n \mathcal{T}_{\textsc{Query}}(m,d,k_{i,t})+\mathcal{T}_{\textsc{Update}}\cdot |\cup_{i\in[n]}S_{i,\mathrm{fire}}(t)|+d\sum_{i\in[n]}k_{i,t}
\end{align*}
The first term ${\cal T}_1 = \sum_{i=1}^n \mathcal{T}_{\textsc{Query}}(m,d,k_{i,t})$ corresponds to the running time of querying the active neuron set $S_{i,\mathrm{fire}}(t)$ for all training samples $i\in[n]$. With the first result in Theorem~\ref{thm:correlation_tree_data_structure}, the complexity is bounded by $O(m^{4/5}n\log m)$.

The second term ${\cal T}_2 = \mathcal{T}_{\textsc{Update}}\cdot |\cup_{i\in[n]}S_{i,\mathrm{fire}}(t)|$ corresponds to updating $w_r$ in the high-dimensional search data-structure (Line~\ref{lin:weight_update}). Again with the first result in Theorem~\ref{thm:correlation_tree_data_structure}, we have $\mathcal{T}_{\textsc{Update}}=O(n(d+\log m))$. Combining with the fact that $|\cup_{i\in[n]}S_{i,\mathrm{fire}}(t)|\leq |\cup_{i\in[n]}S_{i,\mathrm{fire}}(0)|\leq O(m^{4/5}n)$, the second term is bounded by $O(m^{4/5}n^2d)$.

The third term is the time complexity of gradient calculation restricted to the set $S_{i,\mathrm{fire}}(t)$. With the bound on $\sum_{i\in[n]}k_{i,t}$ (Lemma~\ref{lem:bound_fire_neurons_formal}), 
we have $d\sum_{i\in[n]}k_{i,t}\leq O(m^{4/5}nd)$

Putting them together, we have
\begin{align*}
    {\cal T} 
    \leq & ~ O( m^{4/5} n \log m ) + O( m^{4/5} n^2 d ) + O(m^{4/5} n d ) \\
    = & ~ O(m^{4/5} n^2d)
\end{align*}
Thus, we complete the proof.
\end{proof}

\subsection{Data Preprocessing}

In this section, we describe a similar version of training algorithm aforementioned but it uses WTree to preprocess data points based on weights.

\begin{theorem}[Running time part, formal version of Theorem~\ref{thm:running_time_weight}]\label{thm:running_time_weight_formal}
Given $n$ data points in $\R^d$. Running gradient descent algorithm (Algorithm~\ref{alg:nn_wtree_app}) on $\mathrm{2NN}(m,b=\sqrt{0.4\log m})$, the expected per-iteration running time of initializing $\tilde{S}_{r,\mathrm{fire}},S_{i,\mathrm{fire}}$ for $r\in[m],i\in[n]$ is $O(m^{4/5}n\cdot\log n)$. The cost per-iteration of the training algorithm is $O(m^{4/5} n^2 d)$.
\end{theorem}
\begin{proof}
We analyze the initialization and training parts separately.

{\bf Initialization}
From Line~\ref{lin:init_b} to Line~\ref{lin:init_e}, the sets $\tilde{S}_{r,\mathrm{fire}},S_{i,\mathrm{fire}}$ for $r\in [m],i\in[n]$ are initialized. For each $r\in[m]$, we need to query the data structure the set of data points $x$'s such that $\sigma_b(w_r(0)^\top x)>0$. Hence the running time of this step is
\begin{align*}
    \sum_{r=1}^m \mathcal{T}_{\textsc{Query}}(n,\tilde{k}_{r,0}) & ~=~O(\sum_{r=1}^m \tilde{k}_{r,0}\cdot\log n) \\
    &~=~O(\sum_{i=1}^n k_{i,0}\cdot\log n) \\
    &~=~O(m^{4/5}n\cdot\log n)
\end{align*}
where the second step follows from $\sum_{r=1}^m \tilde{k}_{r,0} = \sum_{i=1}^n k_{i,0}$.

{\bf Training}
Consider training the neural networkfor $T$ steps. For each step, first notice that the forward and backward computation parts (Line~\ref{lin:forward_b} - Line~\ref{lin:backward_e}) are the same as previous algorithm. The time complexity is $O(m^{4/5}n)$.

We next show that maintaining $\tilde{S}_{r,\mathrm{fire}},r\in [m]$ and $S_{i,\mathrm{fire}},i\in [n]$ (Line~\ref{lin:data_update_b} - Line~\ref{lin:data_update_e}) takes $O(m^{4/5} n d)$ time. For each fired neuron $r\in [m]$, we first remove the indices of data in the sets $S_{i,\mathrm{fire}}$, which takes time
\begin{align*}
    O(1)\cdot \sum_{r\in \cup _{i\in[n]}S_{i,\mathrm{fire}}} \tilde{k}_{r,t}~=~O(1)\cdot \sum_{r=1}^m \tilde{k}_{r,t}=O(m^{4/5}n)
\end{align*}
Then, we find the new set of $x$'s such that $\sigma_b(\langle w_r(t+1),x\rangle)>0$ by querying the correlation tree data structure. The total ruunning time for all fired neurons is
\begin{align*}
    \sum_{r\in \cup _{i\in[n]}S_{i,\mathrm{fire}}} \mathcal{T}_{\textsc{Update}}(n,d)+ \mathcal{T}_{\textsc{Query}}(n,\tilde{k}_{r,t+1}) & \lesssim m^{4/5}n^2(d+\log m) + \sum_{r\in \cup _{i\in[n]}S_{i,\mathrm{fire}}} \tilde{k}_{r,t+1}\cdot \log n \\ 
    & =O(m^{4/5} n^2 d)
\end{align*}
Then, we update the index sets $S_{i,\mathrm{fire}}$ in time $O(m^{4/5}n)$. Therefore, each training step takes $O(m^{4/5} n^2 d)$ time, which completes the proof.
\end{proof}

\end{document}